\documentclass[wcp]{jmlr}

\usepackage{longtable}% for long tables

\usepackage{booktabs}
\makeatletter
\let\Ginclude@graphics\@org@Ginclude@graphics 
\makeatother

\jmlrvolume{157}
\jmlryear{2021}
\jmlrworkshop{ACML 2021}

\usepackage[utf8]{inputenc} % allow utf-8 input
\usepackage[T1]{fontenc}    % use 8-bit T1 fonts
\usepackage{hyperref}       % hyperlinks
\usepackage{url}            % simple URL typesetting
\usepackage{booktabs}       % professional-quality tables
\usepackage{amsfonts}       % blackboard math symbols
\usepackage{nicefrac}       % compact symbols for 1/2, etc.
\usepackage{microtype}      % microtypography
\usepackage{graphicx}
\usepackage{booktabs} % for professional tables
\usepackage{wrapfig}
\usepackage{multirow}
\usepackage{wrapfig}
\usepackage{algorithm}
\usepackage{algorithmic}
\newcommand{\uh}{\ensuremath{\hat{u}}}
\newcommand{\vh}{\ensuremath{\hat{v}}}

\newcommand{\us}{\ensuremath{u^\ast}}
\newcommand{\vs}{\ensuremath{v^\ast}}

\newcommand{\uo}{\ensuremath{\overline{u}}}
\newcommand{\vo}{\ensuremath{\overline{v}}}
\newcommand{\wo}{\ensuremath{\overline{w}}}
\newcommand{\Uc}{\ensuremath{\mathcal{U}}}
\newcommand{\Vc}{\ensuremath{\mathcal{V}}}

\newcommand{\Dc}{\ensuremath{\mathcal{D}}}
\newcommand{\fu}{\ensuremath{\nabla_u f}}
\newcommand{\fv}{\ensuremath{\nabla_v f}}
\newcommand{\fuu}{\ensuremath{\nabla_{uu}^2 f}}
\newcommand{\fuv}{\ensuremath{\nabla_{uv}^2 f}}
\newcommand{\fvu}{\ensuremath{\nabla_{vu}^2 f}}
\newcommand{\fvv}{\ensuremath{\nabla_{vv}^2 f}}

\newcommand{\gv}{\ensuremath{\nabla_v g}}

\newcommand{\guv}{\ensuremath{\nabla_{uv}^2 g}}

\newcommand{\gvv}{\ensuremath{\nabla_{vv}^2 g}}

\usepackage{makecell}
\usepackage{amssymb}
\usepackage{amsmath}
\usepackage{amsthm}
\newtheorem{thm}{Theorem}
\newtheorem{lem}[thm]{Lemma}

\newtheoremstyle{TheoremRep}
        {\topsep}{\topsep}              %%% space between body and thm
        {\itshape}                      %%% Thm body font
        {}                              %%% Indent amount (empty = no indent)
        {\bfseries}                     %%% Thm head font
        {.}                             %%% Punctuation after thm head
        { }                             %%% Space after thm head
        {\thmname{#1}\thmnote{ \bfseries #3}}%%% Thm head spec
\theoremstyle{TheoremRep}
\newtheorem{theoremrep}{Theorem}
\newtheorem{lemmarep}[theoremrep]{Lemma}
\newtheorem{corollaryrep}[theoremrep]{Corollary}

\title[Penalty Method for Inversion-Free Deep Bilevel Optimization]{Penalty Method for Inversion-Free  Deep Bilevel Optimization}
  \author{\Name{Akshay Mehra} \Email{amehra@tulane.edu}\\
  \Name{Jihun Hamm} \Email{jhamm3@tulane.edu}\\
  \addr Tulane University, New Orleans, LA, USA
 }
\editors{Vineeth N Balasubramanian and Ivor Tsang}
\begin{document}
\maketitle
\begin{abstract}
Solving a bilevel optimization problem is at the core of several machine learning problems such as hyperparameter tuning, data denoising, meta- and few-shot learning, and training-data poisoning. 
Different from simultaneous or multi-objective optimization, the steepest descent direction for minimizing the upper-level cost in a bilevel problem requires the inverse of the Hessian of the lower-level cost. 
In this work, we propose a novel algorithm for solving bilevel optimization problems based on the classical penalty function approach. 
Our method avoids computing the Hessian inverse and can handle constrained bilevel problems easily.  
We prove the convergence of the method under mild conditions and show that the exact hypergradient is obtained asymptotically. 
Our method's simplicity and small space and time complexities enable us to effectively solve large-scale bilevel problems involving deep neural networks. 
We present results on data denoising, few-shot learning, and training-data poisoning problems in a large-scale setting. 
Our results show that our approach outperforms or is comparable to previously proposed methods based on automatic differentiation and approximate inversion in terms of accuracy, run-time, and convergence speed.
\end{abstract}
\begin{keywords}
%List of keywords separated by semicolon.
Bilevel optimization; data denoising; few-shot learning; data poisoning.
\end{keywords}

%%%%%%%%%%%%%%%%%%%%%%%%%%%%%%%%%%%%%%%%%%%%%%%%%%%%%%%%%%%%%%%%%%%%%%%%%%%%%%%%%%%%%%%%%%%%%%%%%%%%%%%%%%%%%%%%%%%%%
\section{Introduction}\label{sec:intro}
%%%%%%%%%%%%%%%%%%%%%%%%%%%%%%%%%%%%%%%%%%%%%%%%%%%%%%%%%%%%%%%%%%%%%%%%%%%%%%%%%%%%%%%%%%%%%%%%%%%%%%%%%%%%%%%%%%%%%
Bilevel optimization problems appear in the fields of study involving a competition between two parties or two objectives. 
Particularly, a bilevel problem arises if one party makes its choice first affecting the optimal choice for the second party, known as the Stackelberg model [\cite{von2010market}]. 
The general form of a bilevel optimization problem is as follows
\vspace{-0.2cm}
\begin{equation}\label{eq:bilevel}
\min_{u \in \Uc}\; f(u,v^*)\;\;\;\mathrm{s.t.}\;\;v^* = \arg\min_{v\in \Vc(u)}\;g(u,v)
\vspace{-0.2cm}
\end{equation}
The `upper-level' problem ${\min_{u \in \Uc} f(u,v^*)}$ is a usual minimization problem except that $v^*$ is constrained to be the solution to the `lower-level' problem ${\min_{v\in \Vc(u)} g(u,v)}$ which is in turn dependent on $u$ (see [\cite{bard2013practical}] for a review of bilevel optimization). 
In this work, we propose and analyze a new algorithm for solving bilevel problems based on the classical penalty function approach. We demonstrate the effectiveness of our algorithm on several important machine learning applications including
gradient-based hyperparameter tuning [\cite{domke2012generic,maclaurin2015gradient,luketina2016scalable,pedregosa2016hyperparameter,franceschi2017forward,franceschi2018bilevel,lorraine2020optimizing}], data denoising by importance learning [\cite{liu2016classification,yu2017transfer,ren18l2rw,han2018coteaching}], meta/few-shot learning [\cite{ravi2016optimization,santoro2016one,vinyals2016matching,franceschi2017forward,mishra2017simple,snell2017prototypical,franceschi2018bilevel, rajeswaran2019meta}], and training-data poisoning [\cite{mei2015using,munoz2017towards,koh2017understanding,shafahi2018poison}].
%Below, we describe the bilevel formulation for these machine learning applications.

{\bf Gradient-based hyperparameter tuning.} Hyperparameter tuning is essential 
%Searching for optimal hyperparameters is an indispensable step 
for any learning problem and grid search is a popular method when domain of the hyperparameters is a discrete set or a range. However, when losses are differentiable functions of the hyperparameter(s), a continuous bilevel optimization problem can help find the optimal hyperparameters. Let $u$ and $w$ be hyperparameter(s) and parameter(s) for a class of learning algorithms, $h(x;u,w)$ be the hypothesis, $L_\mathrm{val}(u,w) = \frac{1}{N_\mathrm{val}}\sum_{(x_i,y_i) \in \Dc_\mathrm{val}} l(h(x_i;u,w),y_i)$ and $L_\mathrm{train}(u,w)=\frac{1}{N_\mathrm{train}}\sum_{(x_i,y_i) \in \Dc_\mathrm{train}} l(h(x_i;u,w),y_i)$ be the loss on validation and training sets, respectively. Then the best hyperparameter(s) $u$ is the solution to
\vspace{-0.1cm}
\begin{equation}\label{eq:hyperparameter learning}
\vspace{-0.2cm}
\min_{u}\; L_\mathrm{val}(u,w^*) \;\;\mathrm{s.t.}\;\;w^* = \arg\min_{w}\; L_\mathrm{train}(u,w).
\end{equation} 

{\bf Data denoising by importance learning.} Most learning algorithms assume that the training set is an i.i.d. sample from the same distribution as the test set. However, if train and test distributions are not identical or if the training set is corrupted by noise or modified by adversaries, this assumption is violated. In such cases, changing the importance of each training example, before training, can reduce the discrepancy between the two distributions. For example, importance of the examples from the same distribution can be up-weighted in comparison to other examples. Determining the importance of each training example can be formulated as a bilevel problem. Let $u$ be the vector of non-negative importance values ${u=[u_1,\cdots, u_{N}]^T}$ where $N$ is the number of training examples, $w$ be the parameter(s) of a classifier $h(x;w)$. Assuming access to a small set of validation data, from the same distribution as test data, $L_{\mathrm{val}}(u,w) = \frac{1}{N_\mathrm{val}}\sum_{(x_i,y_i) \in \Dc_\mathrm{val}} l(h(x_i;u,w),y_i)$ be the loss on validation set and ${L_\mathrm{w\_train}(u,w) :=\frac{1}{\sum_i u_i}\sum_{(x_i,y_i) \in \Dc_\mathrm{train}} u_i l(h(x_i;w),y_i)}$ be the weighted training error. The problem of learning the importance of each training example is as follows
\begin{equation}
\label{eq:importance learning}    
\min_{u} \; L_\mathrm{val}(u,w^*)\;\;
\mathrm{s.t.} \;\; w^* = \arg\min_w  \;L_\mathrm{w\_train}(u,w).
\end{equation}

{\bf Meta-learning.}
This problem involves learning a prior on the hypothesis classes (a.k.a. inductive bias) for a given set of tasks. Few-shot learning is an example of meta-learning, where a learner is trained on several related tasks, during the meta-training phase, so that it generalizes well on unseen (but related) tasks during the meta-testing phase. An effective approach to this problem is to learn a common representation for various tasks and train task specific classifiers over this representation. Let $T$ be the map that takes raw features to a common representation ${T: \mathcal{X} \to \mathbb{R}^d}$ for all tasks and $h_i$ be the classifier for the $i$-th task, ${i \in \{1,\cdots,M\}}$ where $M$ is the total number of tasks for training. The goal is to learn both the representation map $T(\cdot\;;u)$ parameterized by $u$ and the set of classifiers $\{h_1,\cdots,h_M\}$ parameterized by ${w=\{w_1,\cdots,w_M\}}$. Let ${L_\mathrm{val}(u,w_i):= \frac{1}{N_\mathrm{val}} \sum_{(x_i,y_i) \in \Dc_\mathrm{val}} l(h_i(T(x_i;u);w_i),y_i)}$ be the validation loss of task $i$ and ${L_\mathrm{train}(u,w_i)}$ be the training loss defined similarly. Then the bilevel problem for few-shot learning is as follows
\begin{equation}
\label{eq:meta-learning}
\min_{u} \;\sum_i L_\mathrm{val}(u,w_i^*)  \;\; \mathrm{s.t.}\;\;
 w_i^* = \arg\min_{w_i}\;  L_\mathrm{train}(u,w_i),\;i=1,\cdots,M.
\end{equation}
At test time the common representation ${T(\cdot\;;u)}$ is kept fixed and the classifiers $h^{'}_{i}$ for the new tasks are trained i.e. ${\min_{w^{'}_i}\;L_\mathrm{test}(u, w^{'}_i) \;\; i=1,\cdots,N}$ where $N$ is the total number of tasks for testing.

{\bf Training-data poisoning.}
This problem refers to the setting in which an adversary can modify the training data so that the model trained on the altered data performs poorly/differently compared to one trained on the unaltered data. Attacker adds one or more `poisoned' examples ${u=\{u_1,\cdots,u_M\}}$ to the original training data ${X=\{x_1,\cdots,x_N\}}$ i.e., ${X' = X \bigcup u}$ with arbitrary labels. Additionally, to evade detection, an attacker can generate poisoned images starting from existing clean images (called base images) with a bound on the maximum perturbation allowed. Let ${L_\mathrm{poison}(u,w) := \frac{1}{N} \sum_{(x_i,y_i) \in X' \times Y'} l(h(x_i;u,w),y_i)}$ be the loss on the poisoned training data, $\epsilon$ be the bound on the maximum perturbation allowed for poisoned points and the validation set consists of target images that an attacker wants the model to misclassify. Then the problem of generating poisoning data is as follows
\begin{equation}\label{eq:clean-label-poisoning}
\min_{u} \;L_\mathrm{val}(u,w^*) \; \mathrm{s.t.} \;
\; w^* = \arg\min_{w}\; L_\mathrm{poison}(u,w)\; \mathrm{and}\; \|x^i_\mathrm{base} - u^i\|_2 < \epsilon \; \mathrm{for}\; i = 1, ..., M .
\end{equation}

{\bf Challenges of deep bilevel optimization.}
General bilevel optimization problems cannot be solved using simultaneous optimization of the upper- and lower-level cost and are in fact, shown to be NP-hard even in cases with linear upper-level and quadratic lower-level functions [\cite{bard1991some}]. 
To make the analyses tractable many previous works require assumptions such as convexity of the lower-level objective and lower-level solution set being a singleton\footnote{Recently, \cite{liu2021value} presented an analysis of bilevel problems without requiring these assumptions.}. 
%Although theoretically sound, the method's convergence to a global solution in problems with multiple solutions cannot be practically guaranteed, especially when using local methods such as gradient descent.  
Moreover, for solving bilevel problems involving deep neural networks, with millions of variables, only first-order methods such as gradient descent are feasible. 
However, the steepest descent direction (Hypergradient in Sec.~\ref{sec:hypergradient}) using the first-order methods for bilevel problems requires the computation of the inverse Hessian--gradient product. 
Since direct inversion of the Hessian is impractical even for moderate-sized problems, previous approaches approximate the hypergradient using either forward/reverse-mode differentiation [\cite{maclaurin2015gradient,franceschi2017forward,shaban2018truncated}] or by approximately solving a linear system [\cite{domke2012generic,pedregosa2016hyperparameter,rajeswaran2019meta,lorraine2020optimizing}]. However, some of these approaches have high space and time complexities
which can be problematic, especially for deep learning settings.

{\bf Contributions.}
    We propose an algorithm (Alg.~\ref{alg:main}) based on the classical penalty function approach for solving large-scale bilevel optimization problems. We prove convergence of the method (Theorem~\ref{thm:main}) and present its complexity analysis showing that it has linear time and constant space complexity (Table~\ref{tbl:complexity}). This makes our approach superior to forward/reverse-mode differentiation and similar to the approximate inversion-based methods. 
    The small space and time complexities of the method make it an effective solver for large-scale bilevel problems involving deep neural networks as shown by our experiments on data denoising, few-shot learning, and training-data poisoning problems. In addition to being able to solve constrained problems, our method performs competitively to the state-of-the-art methods on simpler problems (with convex lower-level cost) and significantly outperforms other methods on complex problems (with non-convex lower-level cost), in terms of accuracy (Sec.~\ref{sec:experiments}), convergence speed (Sec.~\ref{sec:convergence_speed}) and run-time (Appendix \ref{sec:impact_of_T}).
\\\\
The rest of the paper is organized as follows. We present and analyze the main algorithm in Sec.~\ref{sec:main}, present experiments in Sec.~\ref{sec:experiments}, and conclude in Sec.~\ref{sec:conclusion}. Proofs, experimental settings, and additional results are presented in the supplementary material. The scripts used to generate the results in this paper are available at \url{https://github.com/jihunhamm/bilevel-penalty}.

%%%%%%%%%%%%%%%%%%%%%%%%%%%%%%%%%%%%%%%%%%%%%%%%%%%%%%%%%%%%%%%%%%%%%%%%%%%%%%%%%%%%%%%%%%%%%%%%%%%%%%%%%%%%%%%%%%%%%
\section{Inversion-Free Penalty Method}\label{sec:main}
%%%%%%%%%%%%%%%%%%%%%%%%%%%%%%%%%%%%%%%%%%%%%%%%%%%%%%%%%%%%%%%%%%%%%%%%%%%%%%%%%%%%%%%%%%%%%%%%%%%%%%%%%%%%%%%%%%%%%
We assume the upper- and lower-level costs $f$ and $g$ are twice continuously differentiable and the upper-level constraint function $h$ is continuously differentiable in both $u$ and $v$. Let $\fu$ and $\fv$ denote the gradient vectors, $\fuv$ denote the Jacobian matrix $\left[\frac{\partial^2 f}{\partial u_i\partial v_j}\right]$, and $\fvv$ denote the Hessian matrix $\left[\frac{\partial^2 f}{\partial v_i\partial v_j}\right]$. Following previous works, we also assume the lower-level solution $\vs(u):=\arg\min_v\; g(u,v)$ is unique for all $u$ and $\gvv$ is invertible everywhere. 

\subsection{Background}
A bilevel problem is a constrained optimization problem with the lower-level optimality $v^*(u)=\arg\min_v g(u,v)$ being a constraint. Additionally, it can also have other constraints in the upper- and lower-level problems\footnote{Problems with lower-level constraints are less common in machine learning and are left for future work.}.
\begin{equation}\label{eq:bilevel_with_inequality_constraints}
    \min_{u}\;\;f(u,v^*),\;\;\mathrm{s.t.} \;\; h(u,v^*) \leq 0\;\; \mathrm{and}\; \;{v^*(u)=\arg\min_v g(u,v)}.
\end{equation}
Inequality constraints $h(u,v)\leq 0$ can be converted into equality constraints, e.g., $h(u,v) + s^2 = 0$ by using a slack variables $s$, resulting in the following problem
\begin{equation} \label{eq:bilevel_with_equality_constraints}
    \min_{u}\;\;f(u,v^*),\;\;\mathrm{s.t.} \;\; h(u,v^*) = 0\;\; \mathrm{and}\; \;{v^*(u)=\arg\min_v g(u,v)}.
\end{equation}
The assumption about the uniqueness of the lower-level solution for each $u$ allows us to convert the bilevel problem in Eq.~(\ref{eq:bilevel_with_equality_constraints}) into the following single-level constrained problem:
\begin{equation} \label{eq:bilevel_with_constraints_single_reformulation}
\min_{u,v}\;\;f(u,v),\;\;\mathrm{s.t.}\;\; h(u,v) = 0\;\; \mathrm{and}\; \;{\gv = 0}.
\end{equation}

For general bilevel problems, Eq.~(\ref{eq:bilevel_with_equality_constraints}) and Eq.~(\ref{eq:bilevel_with_constraints_single_reformulation}) are not the same [\cite{dempe2012bilevel}]. But, for simpler problems where the lower-level cost $g$ is convex in $v$ for each $u$, the lower-level solution is unique for each $u$ and the upper-level problem does not have any additional constraints, the problems in Eq.~(\ref{eq:bilevel_with_constraints_single_reformulation}) and Eq.~(\ref{eq:bilevel_with_equality_constraints}) are equivalent.

{\bf Hypergradient for bilevel optimization.}\label{sec:hypergradient}
For such simpler problems, we can use the gradient-based approaches on the single-level problem in Eq.~(\ref{eq:bilevel_with_constraints_single_reformulation}) to compute the total derivative $\frac{df}{du}(u,\vs(u))$, also known as the hypergradient. By the chain rule, we have $\frac{df}{du} = \fu + \frac{dv}{du}\cdot\fv$ at $(u,\vs(u))$. Even if $\vs(u)$ cannot be found explicitly, we can still compute ${\frac{dv}{du}}$ using the implicit function theorem. As ${\gv=0}$ at ${v=\vs(u)}$ and $\gvv$ is invertible, we get ${du \cdot \guv + dv \cdot \gvv =0}$, and ${\frac{dv}{du} = -\guv (\gvv)^{-1}}$. Thus the hypergradient is
\begin{equation}\label{eq:hypergradient}
\frac{df}{du} = \fu -\guv (\gvv)^{-1} \fv \;\; \mathrm{at} \;\;(u,\vs(u)).
\end{equation}
Existing approaches [\cite{domke2012generic,maclaurin2015gradient,pedregosa2016hyperparameter,franceschi2017forward,shaban2018truncated,rajeswaran2019meta,lorraine2020optimizing}] can be viewed as implicit methods of approximating the hypergradient, with distinct efficiency and iteration complexity trade-offs [\cite{grazzi2020iteration}].

\subsection{Penalty function approach}
The penalty function method is a well-known approach for solving constrained optimization problems as in Eq.~(\ref{eq:bilevel_with_constraints_single_reformulation}) [\cite{bertsekas1997nonlinear}]. It has been previously applied to solve bilevel problems described in Eq.~(\ref{eq:bilevel_with_inequality_constraints}) under strict assumptions and only high-level descriptions of the algorithm were presented [\cite{aiyoshi1984solution,ishizuka1992double}]. 
The corresponding penalty function that is minimized to solve Eq.~(\ref{eq:bilevel_with_constraints_single_reformulation}) has the form $\tilde{f}(u,v;\gamma):= f(u,v) + \frac{\gamma_k}{2}(\Psi(\|h(u,v)\|) + \|\gv(u,v)\|^2)$ which is the sum of the original cost $f$ and the penalty terms for constraint satisfaction and first-order stationarity of the lower-level problem.
The function $\Psi$ is an interior or exterior penalty function depending on whether $h(u,v)$ is an equality or inequality constraint.
%Using this penalty function we can solve the problem in Eq.~(\ref{eq:bilevel_with_constraints_single_reformulation}) as follows
The following result for obtaining a solution to the inequality constrained problem in Eq.~(\ref{eq:bilevel_with_inequality_constraints}) by solving Eq.~(\ref{eq:bilevel_with_constraints_penalty_formulation}) is known. 
\begin{equation} \label{eq:bilevel_with_constraints_penalty_formulation}
(\uh_k,\vh_k) = \arg\min_{u,v} \tilde{f}(u,v;\gamma_k) = \arg\min_{u,v} f(u,v) + \frac{\gamma_k}{2}(\Psi(\|h(u,v)\|) + \|\gv(u,v)\|^2).
\end{equation}

\begin{thm}[{Simplified Theorem 8.3.1 [\cite{bard2013practical}]}]\label{thm:strict}
%for Eq.~(\ref{eq:bilevel_with_constraints_single_reformulation})

Assume $f$ and $g$ are convex in $v$ for any fixed $u$, 
Let $\{\gamma_k\}$ be any positive ($\gamma_k>0$) and divergent ($\gamma_k \to \infty$) sequence.
If $\{(\uh_k,\vh_k)\}$ is the corresponding sequence of the {\bf optimal} solutions to Eq.~(\ref{eq:bilevel_with_constraints_penalty_formulation}), then the sequence $\{(\uh_k,\vh_k)\}$ has limit points any one of which is a solution
to Eq.~(\ref{eq:bilevel_with_inequality_constraints}).
\end{thm}

\subsection{Our algorithm}
Theorem~\ref{thm:strict} presents a strong result, however, it is not very practical, especially for bilevel problems involving deep neural networks due to the following reasons. Firstly, the minimizer $(\uh_k,\vh_k)$ for Eq.~(\ref{eq:bilevel_with_constraints_penalty_formulation}) cannot be computed exactly for each $\gamma_k$ and it is not computationally possible to increase $\gamma_k \to \infty$. Secondly, the upper- and lower-level costs $f$ and $g$ may not be convex in $v$ for any $u$. 
To overcome some of these limitations and guarantee convergence in deep learning settings, 
%we allow non-convexity of $f$
we allow $\epsilon_k$-optimal (instead of exact) solution to Eq.~(\ref{eq:bilevel_with_constraints_penalty_formulation}) at each $k$.
Our Theorem~\ref{thm:main} below (for equality constrained problems Eq.~(\ref{eq:bilevel_with_equality_constraints})), shows that the solution found by allowing $\epsilon_k$-optimal solution to Eq.~(\ref{eq:bilevel_with_constraints_penalty_formulation}) converges to a KKT point of Eq.~(\ref{eq:bilevel_with_constraints_single_reformulation}), assuming that linear independence constraint qualification (LICQ) is satisfied at the optimum (i.e., linear independence of the gradients of the constraints, $h$ and $\gv$). We handle inequality constraints using slack variables and use $\Psi(\|\cdot\|) = \|\cdot\|^2$.
Using the penalty function approach, we propose an algorithm for solving large-scale bilevel problems in Alg.~\ref{alg:main}. 

\begin{thm}\label{thm:main}%(Inspired by Theorem 17.2 of nocedal)
Suppose $\{\epsilon_k\}$ is a positive ($\epsilon_k>0$) and convergent ($\epsilon_k \to 0$) sequence,
$\{\gamma_k\}$ is a positive ($\gamma_k>0$), non-decreasing ($\gamma_1 \leq \gamma_2 \leq \cdots$),
and divergent ($\gamma_k \to \infty$) sequence.
Let $\{(u_k,v_k)\}$ be the sequence of approximate solutions to Eq.~(\ref{eq:bilevel_with_constraints_penalty_formulation}) with tolerance 
$(\nabla_u \tilde{f}(u_k,v_k))^2 + (\nabla_v \tilde{f}(u_k,v_k))^2 \leq \epsilon_k^2$ for all $k=0,1,\cdots$ and LICQ is satisfied at the optimum. 
Then any limit point of $\{(u_k,v_k)\}$ satisfies the KKT conditions of the problem in Eq.~(\ref{eq:bilevel_with_constraints_single_reformulation}).
\end{thm}

\begin{algorithm} 
\caption{Our algorithm for solving bilevel problems (Penalty).} \label{alg:main}
{Input}: $K,T, \{\sigma_k\},\{\rho_{k,t}\}, \gamma_0,\epsilon_0, c_\gamma\textrm{(=1.1)}, c_\epsilon\textrm{(=0.9)}$ \\
{Output}: $(u_{K},v_{T})$\\
Initialize $u_0,v_0$ randomly\\
{Begin}
\begin{algorithmic}
\FOR{$k=0,\;\cdots\;,K\textrm{-}1$}
	%\STATE\hspace{-0.075in}{{\it $v$-update:}}
    %\STATE{$w=v_{t-1}$}
	\WHILE{$\|\nabla_u \tilde{f}\|^2+\|\nabla_v \tilde{f}\|^2 > \epsilon_{k}^2$}
    	\FOR{$t=0,\cdots,T\textrm{-}1$}
            %\STATE{$v$-update:\;\;$v_{t+1} \leftarrow v_{t} - \rho_{k,t} \nabla_v \tilde{f}\;\;\textrm{where}\;\;\nabla_v \tilde{f}=\nabla_v\left[f(u_{k},v_{t}) + \frac{\gamma_{k}}{2} \|\nabla_v g(u_{k},v_{t})\|^2\right]$}
            \STATE{$v_{t+1}\leftarrow v_{t}-\rho_{k,t}\nabla_v\tilde{f}$ (from Eq.~(\ref{eq:bilevel_with_constraints_penalty_formulation}))}
        \ENDFOR
	%\STATE{$\nabla_u \tilde{f} = \nabla_u [f(u_{k},v_{T}) + \frac{\gamma_{k}}{2} \|\nabla_v g(u_{k},v_{T})\|^2]$}		
    %\STATE{$u$-update:\;\;$u_{k+1} \leftarrow u_{k} - \sigma_{k} \nabla_u \tilde{f}\;\;\textrm{where}\;\;\nabla_u \tilde{f} = \nabla_u\left[f(u_{k},v_{T}) + \frac{\gamma_{k}}{2} \|\nabla_v g(u_{k},v_{T})\|^2\right]$}	
    \STATE{$u_{k+1} \leftarrow u_{k} - \sigma_{k} \nabla_u \tilde{f}\;$(from Eq.~(\ref{eq:bilevel_with_constraints_penalty_formulation}))}
    \ENDWHILE
	%\STATE\hspace{-0.075in}{{\it Increase penalty}}    
	%\STATE{Break if max iteration is reached}
    \STATE{$\gamma_{k+1} \leftarrow c_\gamma \gamma_{k},\;\epsilon_{k+1} \leftarrow c_\epsilon \epsilon_{k}$}
\ENDFOR
\end{algorithmic}
\end{algorithm}

The proof for Theorem~\ref{thm:main} is based on the standard proof for penalty function methods (See [\cite{nocedal2006numerical}]) and is presented in the Appendix~\ref{app:proofs}. Alg.~\ref{alg:main} describes our method where we minimize the penalty function in Eq.~(\ref{eq:bilevel_with_constraints_penalty_formulation}), alternatively over $v$ and $u$. This greatly reduces the complexity of solving a bilevel problem since in each iteration we only approximately solve a single-level problem over the penalty function $\tilde{f}$ with guaranteed convergence to KKT point of Eq.~(\ref{eq:bilevel_with_constraints_single_reformulation}). Moreover, for unconstrained problems ($h\equiv0$), Lemma~\ref{lm:main} below shows that the approximate gradient direction $\nabla_u \tilde{f}$, computed from Alg.~\ref{alg:main} becomes the exact hypergradient Eq.~(\ref{eq:hypergradient}) asymptotically. 

\begin{lem}\label{lm:main}
Assume $h \equiv 0$. Given $u$, let $\vh$ be $\vh:=\arg\min_v \tilde{f}(u,v;\gamma)$ from Eq.~(\ref{eq:bilevel_with_constraints_penalty_formulation}). 
Then, $\nabla_u\tilde{f}(u,\vh;\gamma)=\frac{df}{du} (u,\vh)$ as in Eq.~(\ref{eq:hypergradient}). 
\end{lem}

Thus if we find the minimizer $\vh$ of the penalty function for given $u$ and $\gamma$, Alg.~\ref{alg:main} computes the exact hypergradient for unconstrained problems (Eq.~(\ref{eq:hypergradient})) at $(u,\vh)$. 
%Furthermore, under the conditions of Theorem~\ref{thm:strict},
%$\vh(u) \to \vs(u)$ as $\gamma \to \infty$ and we get the exact hypergradient asymptotically. 
A similar theoretical analysis of the penalty method for non-bilevel problems was recently presented in [\cite{9120361}] using a different stopping condition and a weaker constraint quantification condition. Although we have a similar theoretical result, [\cite{9120361}] shows results on small-scale experiments on non-bilevel problems whereas we demonstrate our Alg.~\ref{alg:main} on large-scale bilevel problems appearing in machine learning involving deep neural networks.
% and show significant gains in comparison with other existing bilevel solvers.

{\bf Comparison of Penalty with other algorithms.}
Previous algorithms for solving bilevel optimization problems rely on computing an approximation to the hypergradient using forward/reverse-mode differentiation (FMD/RMD) or approximately solving a linear system (ApproxGrad) (See Appendix~\ref{app:other methods} for a summary). 
Unlike these methods, Penalty does not require an explicit computation of the hypergradient in each iteration and obtains the exact hypergradient asymptotically, leading to better run-times for the experiments (Sec.~\ref{sec:convergence_speed} and Appendix \ref{sec:impact_of_T}). Secondly, Penalty provides an easy way to incorporate upper-level constraints for various problems, unlike other methods which have to rely on projection to satisfy the constraints. Although projection is a reasonable method for convex constraints but for non-convex constraints, where computing the projection is intractable, Penalty has a significant computational advantage. Thirdly, for problems with multiple lower-level solutions, Penalty converges to the optimistic case solution without modification (Appendix~\ref{sec:non-uniqueness}) whereas convergence of other methods is unknown because the hypergradient is dependent on the choice of the particular minimizer of the lower-level cost. Lastly, for unconstrained problems, we show the trade-offs of the different methods for computing the hypergradient in Table~\ref{tbl:complexity} and see that as $T$ (total number of $v$-updates per one hypergradient computation) increases, FMD and RMD become impractical due to $O(cUT)$ time complexity and $O(U+VT)$ space complexity, respectively, whereas ApproxGrad and Penalty, have the same linear time complexity and constant space complexity. 
Recently, \cite{lorraine2020optimizing} proposed an implicit gradient-based method similar to ApproxGrad where the Hessian inverse term in the hypergradient is approximated using the terms of the Neumann series. 
%The space and time complexity of this method is the same as that of Penalty since solving the lower-level problem and computing the Hessian inverse for $T$ steps requires $O(cT)$ time and $O(U+V)$ space. 
Another recent method, BVFIM, \cite{liu2021value}, solves the bilevel problem by using a regularized value function of the lower-level problem in the upper-level objective and then solves a sequence of unconstrained problems. 
The time and space complexity of both these methods are the same as those of Penalty. 
Since the complexity analysis does not show the quality of the hypergradient approximation of these methods, we extensively compare Penalty against existing bilevel solvers (Sec.~\ref{sec:experiments}) and show its superiority on synthetic and real problems. 
We present a detailed comparison between ApproxGrad and Penalty since they have the same complexities and show that Penalty has better performance, run-time, and convergence speed (Fig.~\ref{fig:acc_vs_wct} and Fig.~\ref{fig:effect_of_T} in Appendix~\ref{app:additional_experiments}). 

{\bf Improvements.} Some of the assumptions made for analysis such as unique lower-level solution may not hold in practice. Here, we discuss some techniques to address these issues and improve Alg.~\ref{alg:main}. The first problem is non-convexity of the lower-level cost $g$, which leads to the problem that a local minimum of $\|\gv\|$ can be either minima, maxima, or a saddle point of $g$. To address this we modify the $v$-update in Eq.~(\ref{eq:bilevel_with_constraints_penalty_formulation}) by adding a `regularization' term $\lambda_k g$ to the cost. %so that $v$ finds a minimum of $g$.
The minimization over $v$ becomes ${\min_{v}(\tilde{f} + \lambda_k g)}$. This affects the optimization in the beginning; but as $\lambda_k\to 0$ the final solution remains unaffected with or without regularization.
The second problem is that the tolerance $\nabla_{(u,v)} \tilde{f}(u_k,v_k;\gamma_k) \leq \epsilon_k$ may not be satisfied in a limited time and the optimization may terminate before $\gamma_k$ becomes large enough. The method of multipliers and augmented Lagrangian [\cite{bertsekas1976penalty}] can help the penalty method to find a solution with a finite $\gamma_k$. Thus we add the term $\gv^T \nu$ to the penalty function in Eq.~(\ref{eq:bilevel_with_constraints_penalty_formulation}) to get ${\min_{u,v}(\tilde{f} + \gv^T \nu)}$ and use the method of multipliers to update $\nu$. 
In summary, we use the following update rules.
%\begin{eqnarray*}
$u_{k+1} \leftarrow u_k - \rho \nabla_u ( \tilde{f} + \gv^T \nu_k )$,
$v_{k+1} \leftarrow v_k  - \sigma \nabla_v (\tilde{f} + \gv^T \nu_k + \lambda_k g)$,
$\nu_{k+1} \leftarrow \nu_k + \gamma_k \gv$.
%\end{eqnarray*}
Practically, the improvement due to these changes is moderate and problem-dependent (see Appendix~\ref{app:extensions} for details).

\begin{table*}[t]
  \caption{\small{Complexity analysis of various bilevel methods (Appendix~\ref{app:other methods}) on unconstrained problems. $U$ is the size of $u$, $V$ is the size of $v$, $T$ is the number of $v$-updates per one hypergradient computation. $P$, $p$ and $q$ are
  variables of size $U \times V$, $U \times 1$, $V \times 1$ used to compute the hypergradient. We assume gradient of $f$, $g$, Hessian-vector product $\gvv$ and Jacobian-vector product $\guv$ can be computed in time $c=c(U,V)$. 
  %has $O(V)$ complexity [\cite{pearlmutter1994fast}]. 
  (We use gradient descent as the process for FMD and RMD.)}
  }
  \label{tbl:complexity}
  \centering
  \resizebox{0.85\textwidth}{!}{
   \small
      \begin{tabular}{c|cccc}
        \toprule
        Method & $v$-update & Intermediate updates & Time & Space \\
        \midrule
        \multirow{2}{*}{FMD} & \multirow{2}{*}{$v \leftarrow v -\rho \gv$} & 
        \multirow{2}{*}{$P\leftarrow P(I\textrm{ - }\rho\gvv) -\rho \guv$} & 
        \multirow{2}{*}{$O(cUT)$} & \multirow{2}{*}{$O(UV)$} \\ \\\cline{1-5}
        
        \multirow{2}{*}{RMD} & \multirow{2}{*}{$v \leftarrow v -\rho \gv$} & 
        $p\leftarrow p -\rho \guv\cdot q$ &
        \multirow{2}{*}{\boldmath$O(cT)$} & \multirow{2}{*}{$O(U+VT)$} \\
        & & $q\leftarrow q - \rho\gvv \cdot q$ & & \\\cline{1-5}
        
        \multirow{2}{*}{ApproxGrad} & \multirow{2}{*}{$v \leftarrow v -\rho \gv$} &
        \multirow{2}{*}{$q \leftarrow q -\rho \gvv[\gvv\cdot q - \fv] $} &
        \multirow{2}{*}{\boldmath$O(cT)$} & \multirow{2}{*}{\boldmath$O(U\textrm{+}V)$} \\ \\\cline{1-5}
        
        \multirow{2}{*}{Penalty} & \multirow{2}{*}{$v \leftarrow v -\rho[\fv\textrm{ + }\gamma \gvv\gv]$} & 
        \multirow{2}{*}{{Not required}} & 
        \multirow{2}{*}{\boldmath$O(cT)$} & \multirow{2}{*}{\boldmath$O(U\textrm{+}V)$} \\ \\
        \bottomrule
      \end{tabular}
      }
\end{table*}

%%%%%%%%%%%%%%%%%%%%%%%%%%%%%%%%%%%%%%%%%%%%%%%%%%%%%%%%%%%%%%%%%%%%%%%%%%%%%%%%%%%%%%%%%%%%%%%%%%%
\section{Experiments}\label{sec:experiments}
%%%%%%%%%%%%%%%%%%%%%%%%%%%%%%%%%%%%%%%%%%%%%%%%%%%%%%%%%%%%%%%%%%%%%%%%%%%%%%%%%%%%%%%%%%%%%%%%%%%
In this section, we evaluate the performance of the proposed method (Penalty) on machine learning problems discussed earlier. Since previous bilevel methods in machine learning dealt only with unconstrained problems, we evaluate Penalty on unconstrained problems in Sec.~\ref{sec:synthetic}-Sec.~\ref{sec:experiment2} and show its effectiveness in solving constrained problems in Sec.~\ref{sec:experiment3_1}.
\vspace{-0.3cm}
\subsection{Synthetic problems}\label{sec:synthetic}
We compare Penalty against gradient descent (GD), reverse-mode differentiation (RMD),
and approximate hypergradient method (ApproxGrad) on synthetic examples. We omit the comparison with forward-mode differentiation (FMD) because of its impractical time complexity for larger problems (Table~\ref{tbl:complexity}). GD refers to the alternating minimization: $u \leftarrow u - \rho \nabla_u f$, $v \leftarrow v - \sigma \nabla_v g$. For RMD, we used the version with gradient descent as the lower-level process. For ApproxGrad experiments, we used Adam for all updates including solving the linear system. We found that Adam performs similar to the conjugate gradient method for solving the linear system with enough iterations. Since Adam uses the GPU effectively we used it for all experiments with ApproxGrad with a large number of iterations and added regularization to improve the ill-conditioning when solving the linear systems.
Using simple quadratic surfaces for $f$ and $g$, we compare all the algorithms by observing their convergence as a function of the number of upper-level iterations for a different number of lower-level updates ($T$). We measure the convergence using the Euclidean distance of the current iterate $(u,v)$ from the closest optimal solution $(\us,\vs)$. Since the synthetic examples are not learning problems, we can only measure the distance of the iterates to an optimal solution ($\|(u,v)-(u^\ast,v^\ast)\|_{2}^{2}$). Fig.~\ref{fig:synthetic1} shows the performance of two 10-dimensional examples described in the caption (see Appendix~\ref{app:synthetic examples}). As one would expect, increasing the number $T$ of $v$-updates makes all the algorithms, except GD, better since doing more lower-level iterations makes the hypergradient estimation more accurate but it increases the run time of the methods. However, even for these examples, only Penalty and ApproxGrad converge to the optimal solution and GD and RMD converge to non-solution points. A large $T$ (eg. 100) makes RMD converge too but for large-scale experiments, it's impractical to have such a large T. From Fig.~\ref{fig:synthetic1}(b), we see that Penalty converges even with $T$=1 while ApproxGrad requires at least $T$=10 and RMD needs an even higher $T$ to converge. %, showing that our method approximates the hypergradient accurately even with smaller $T$. 
This translates to smaller run-time for our method since run-time is directly proportional to $T$ (Table~\ref{tbl:complexity}).

\begin{figure}[tb]
\small
\centering
\subfigure{\label{fig:synthetic1_a}\includegraphics[width=0.45\textwidth]{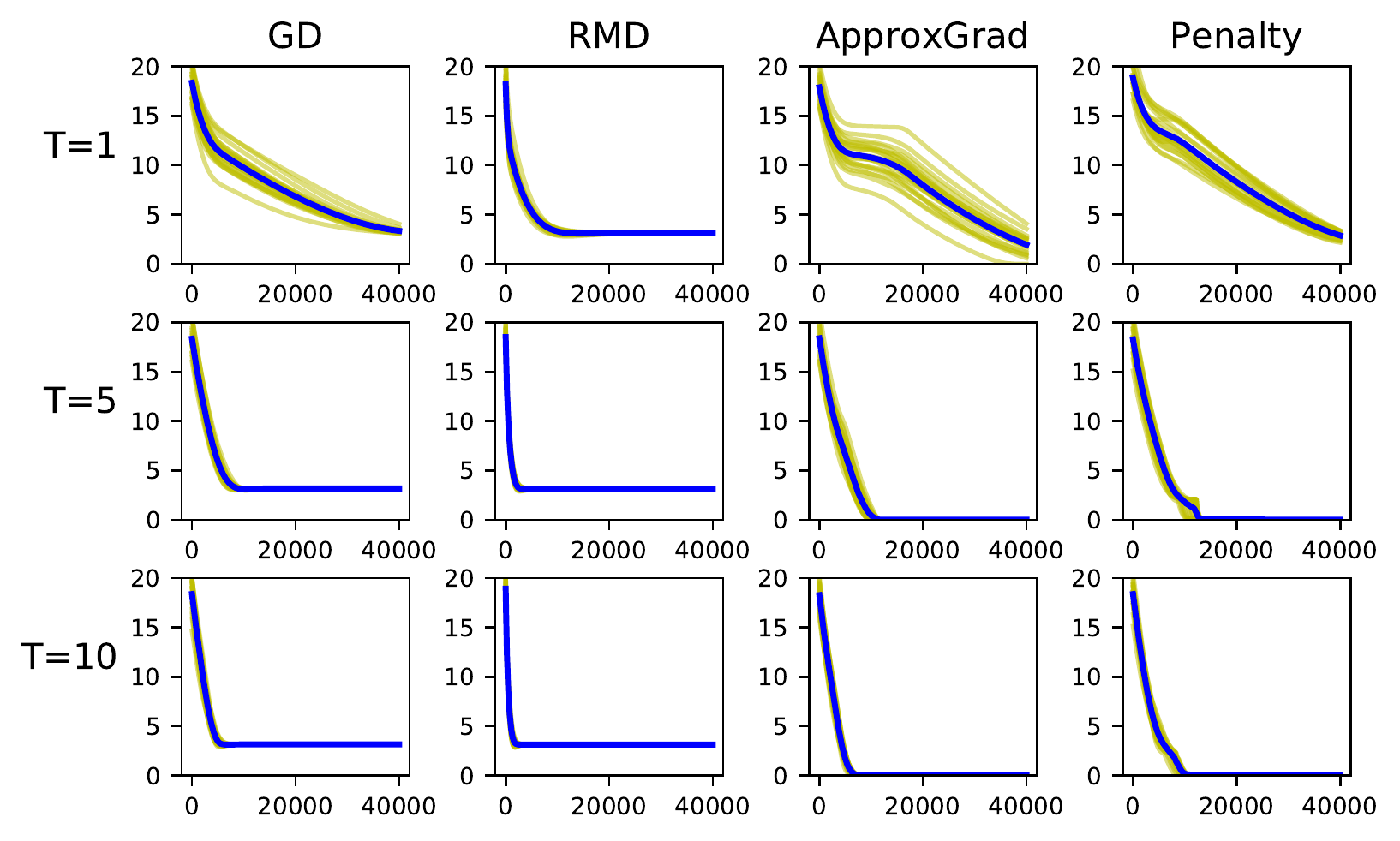}}
\subfigure{\label{fig:synthetic1_b}\includegraphics[width=0.45\textwidth]{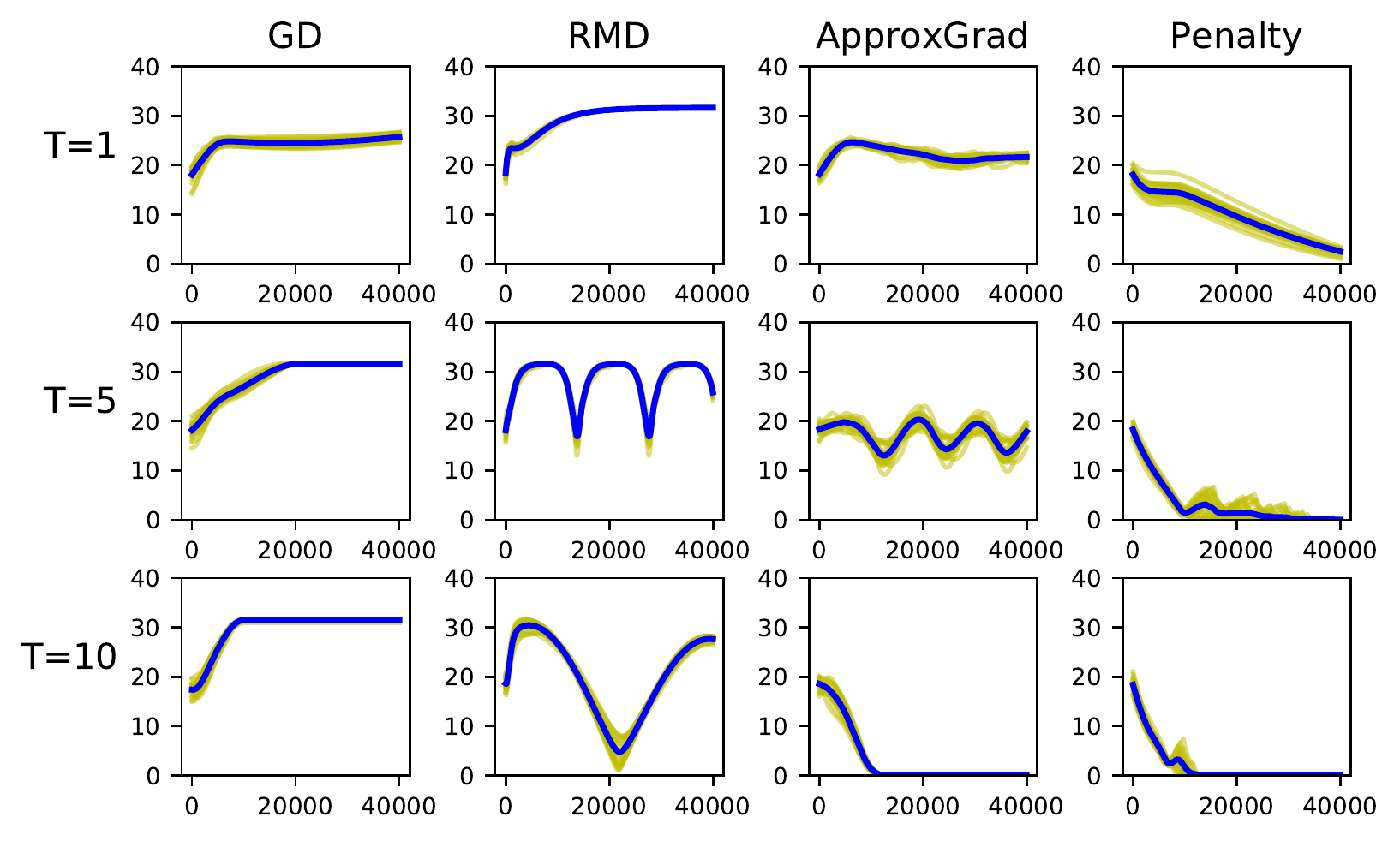}}
\vspace{-0.4cm}
\caption{Convergence of GD, RMD, ApproxGrad, and Penalty vs upper-level epochs (x-axis) on synthetic problems: $f(u,v) = \|u\|^2\textrm{ + }\|v\|^2, g(u,v) = \|1\textrm{-}u\textrm{-}v\|^2$ (a) and $f(u,v) = \|v\|^2\textrm{ - }\|u-v\|^2, g(u,v) = \|u\textrm{-}v\|^2$ (b). Mean curve (blue) is superimposed on 20 independent trials (yellow).
}
\vspace{-1cm}
\label{fig:synthetic1}
\end{figure}

\begin{figure}[tb]
\small
\centering
\subfigure{\includegraphics[width=0.45\textwidth]{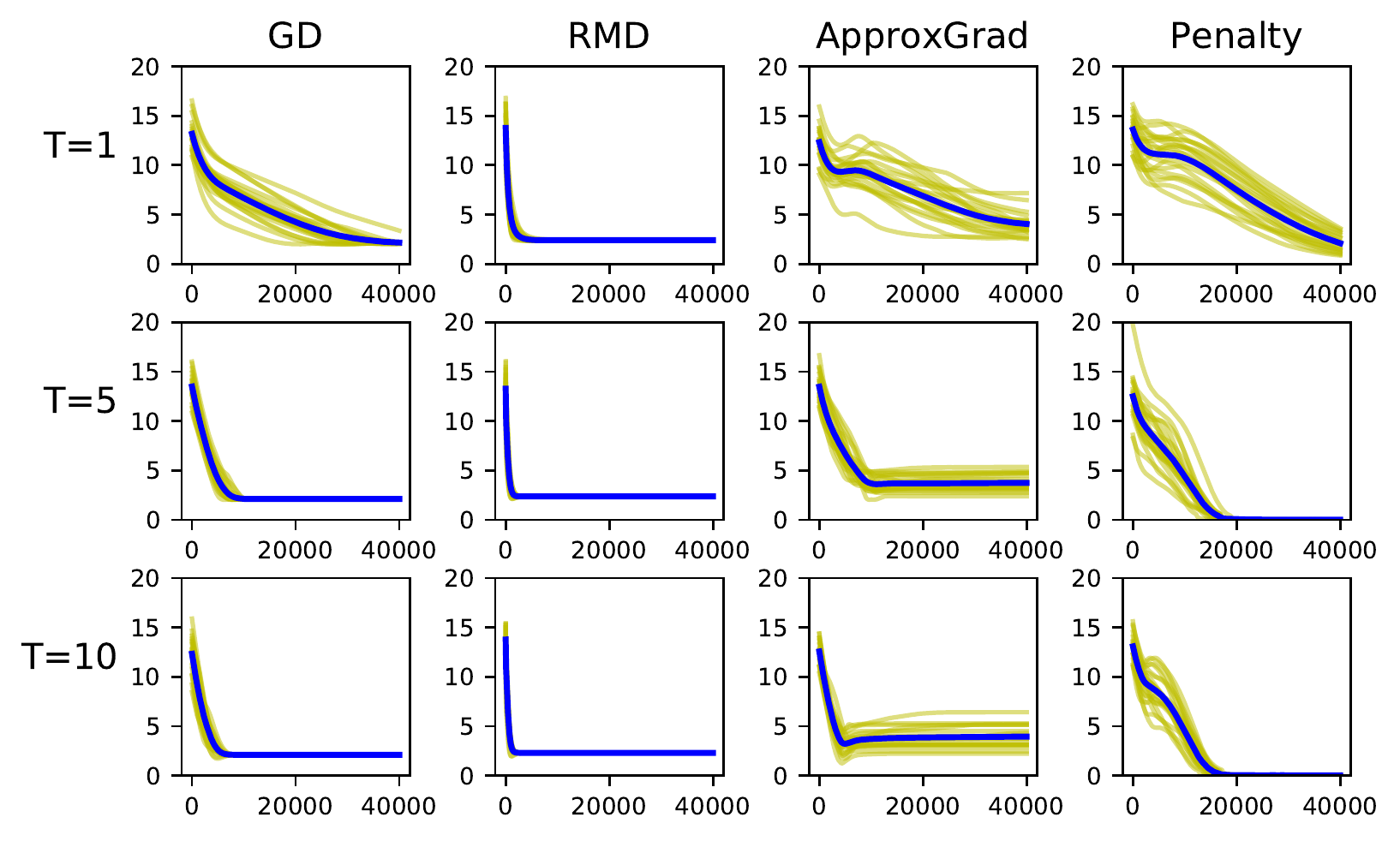}}
\subfigure{\includegraphics[width=0.45\textwidth]{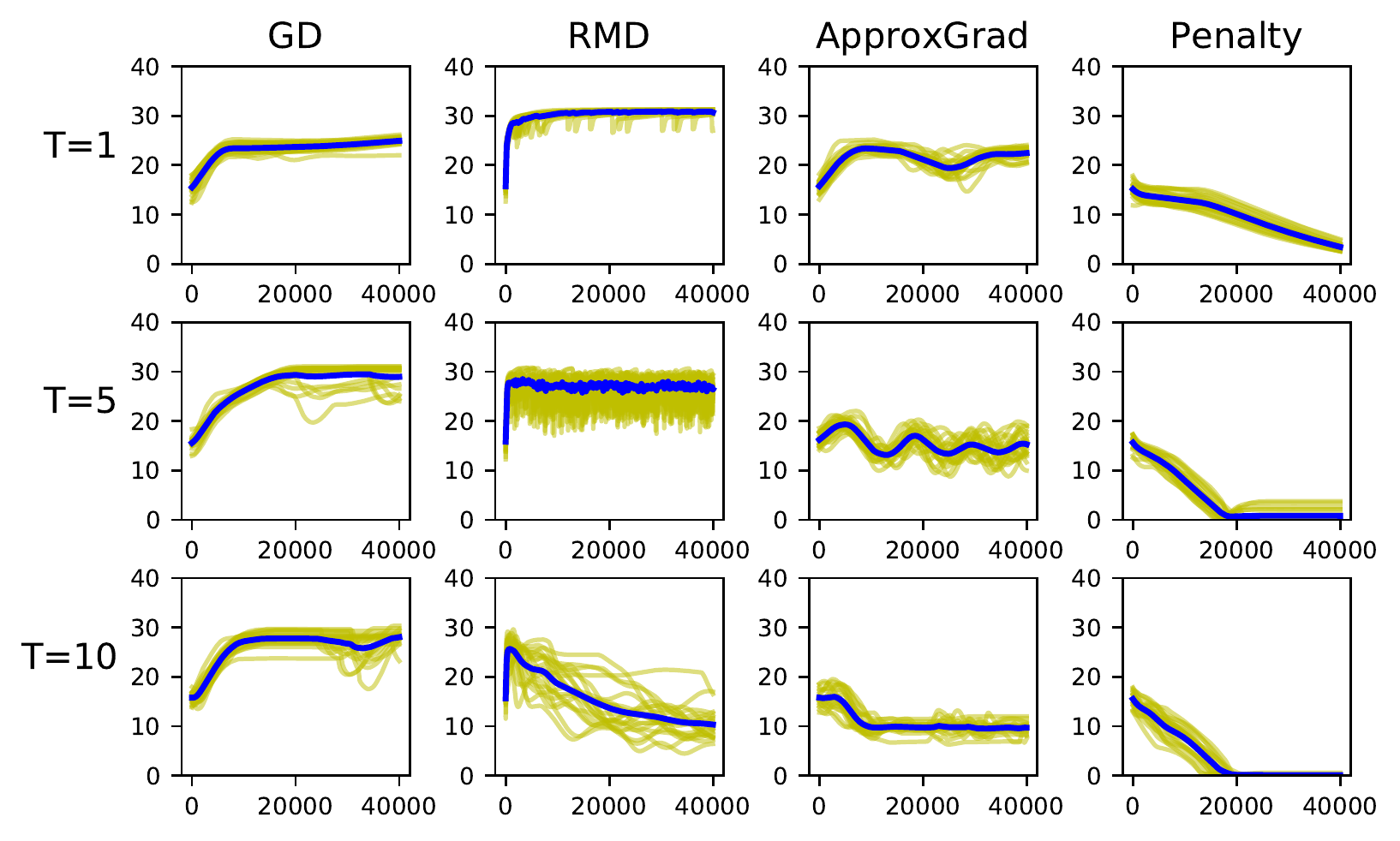}}
\vspace{-0.4cm}
\caption{Convergence of GD, RMD, ApproxGrad, and Penalty vs upper-level epochs (x-axis) for synthetic problems: $f(u,v)\textrm{=}\|u\|^2\textrm{ + }\|v\|^2, g(u,v)\textrm{=}(1\textrm{-}u\textrm{-}v)^TA^TA(1\textrm{-}u\textrm{-}v)$ (a) and $f(u,v)\textrm{=}\|v\|^2\textrm{-}(u\textrm{-}v)^TA^TA(u\textrm{-}v),  g(u,v)\textrm{=}(u\textrm{-}v)^TA^TA(u\textrm{-}v)$ (b), where $A^TA$ is a rank-deficient random matrix. Mean curve (blue) is superimposed on 20 independent trials (yellow).
}
\vspace{-1cm}
\label{fig:synthetic2}
\end{figure}

In Fig.~\ref{fig:synthetic2} we show examples with ill-conditioned or singular Hessian for the lower-level problem ($\gvv$). The ill-conditioning poses difficulty for the methods since the implicit function theorem requires the invertibility of the Hessian at the solution point. Fig.~\ref{fig:synthetic2} shows that only Penalty converges to the true solution even though we added regularization ($\gvv + \lambda I$) while solving the linear system in ApproxGrad. 
%We ascribe the robustness of Penalty to its simplicity. 
%and to the fact that it naturally handles non-uniqueness of the lower-level solution (see Appendix~\ref{sec:non-uniqueness}). 
Additionally, we report the wall clock times for different methods on the four examples tested here in Table~\ref{Table:wall-clock time} in the Appendix. We can see that as we increase the number of lower-level iterations all methods get slower but Penalty is still faster than both RMD and ApproxGrad. Although GD is the fastest but as seen in Fig.~\ref{fig:synthetic1} and~\ref{fig:synthetic2}, GD does not converge to optima since not all bilevel solutions are saddle points. Moreover, for bilevel problems such as in Eq.~(\ref{eq:importance learning}) and Eq.~(\ref{eq:clean-label-poisoning}) the upper-level costs are not directly dependent on the upper-level variable, thus the comparison with GD is omitted for large-scale bilevel problems.

%\begin{table*}[tb]
\begin{wraptable}[9]{r}{9.5cm}
  \vspace{-0.9cm}
  \caption{Test accuracy (\%) of the classifier learned after data denoising using importance learning. (Mean $\pm$ s.d. of 5 runs)}
  \vspace{-0.3cm}
  \label{Table:large_data_denoising}
  \centering
  \small
  \resizebox{0.6\textwidth}{!}{
  \begin{tabular}{c|c|cc|cc}
    \toprule
   Dataset  & & & & \multicolumn{2}{c}{Bilevel Approaches} \\
    (Noise\%) & Oracle & Val-Only & Train+Val & ApproxGrad & Penalty\\
    \midrule
    MNIST (25) & 99.3$\pm$0.1 & 94.6$\pm$0.3 & 83.9$\pm$1.3 & 98.11$\pm$0.08  & {\bf98.89}$\pm$0.04\\
    MNIST (50) & 99.3$\pm$0.1 & 94.6$\pm$0.3 & 60.8$\pm$2.5 & 97.27$\pm$0.15 & {\bf97.51}$\pm$0.07\\
    
    CIFAR10 (25) & 82.9$\pm$1.1 & 70.3$\pm$1.8 & 79.1$\pm$0.8 & 71.59$\pm$0.87 & {\bf79.67}$\pm$1.01\\
    CIFAR10 (50) & 80.7$\pm$1.2 & 70.3$\pm$1.8 & 72.2$\pm$1.8 & 68.08$\pm$0.83 & {\bf79.03}$\pm$1.19\\
    
    SVHN (25) & 91.1$\pm$0.5 & 70.6$\pm$1.5 & 71.6$\pm$1.4 & 80.05$\pm$1.37 & {\bf88.12}$\pm$0.16\\
    SVHN (50) & 89.8$\pm$0.6 & 70.6$\pm$1.5 & 47.9$\pm$1.3 & 74.18$\pm$1.05 & {\bf85.21}$\pm$0.34\\
    \bottomrule
  \end{tabular}
  }
\end{wraptable}

\subsection{Data denoising by importance learning}
\label{sec:experiment1}
We evaluate the performance of Penalty on learning a classifier from a dataset with corrupted labels by posing the problem as an importance learning problem (Eq.~(\ref{eq:importance learning})). The performance of the learned classifier using Penalty, with 20 lower-level updates (T), is evaluated against the following classifiers: \textbf{Oracle}: classifier trained on the portion of the training data with clean labels and the validation data, \textbf{Val-only}: classifier trained only on the validation data, \textbf{Train+Val}: classifier trained on the entire training and validation data, \textbf{ApproxGrad}: classifier trained with our implementation of ApproxGrad, with 20 lower-level and 20 linear system updates. We evaluate the performance on MNIST, CIFAR10, and SVHN datasets with the validation set sizes of 1000, 10000, and 1000 points respectively. We used convolutional neural networks (architectures described in Appendix~\ref{app:experiment1}) at the lower-level for this task. Table~\ref{Table:large_data_denoising} summarizes our results for this problem and shows that Penalty outperforms Val-only, Train+Val, and ApproxGrad by significant margins and in fact performs very close to the Oracle classifier (which is the ideal classifier), even for high noise levels. This demonstrates that Penalty is extremely effective in solving bilevel problems involving several million variables (see Table~\ref{Table:variable_sizes} in Appendix) and shows its effectiveness at handling non-convex problems. Along with the improvement in terms of accuracy over ApproxGrad, Penalty also gives better run-time per upper-level iteration and higher convergence speed leading to a decrease in the overall wall-clock time of the experiments (Fig.~\ref{fig:acc_vs_wct}(a) and Fig.~\ref{fig:effect_of_T}(a) in Appendix~\ref{sec:impact_of_T}).

Next, evaluate Penalty against recent methods [\cite{ren18l2rw,han2018coteaching}] which use a meta-learning-based approach to assigns weights to the examples. We used the same setting as \citeauthor{ren18l2rw}'s uniform flip experiment with 36\% label noise on the CIFAR10 dataset and Wide ResNet  28-10 (WRN-28-10) model (see Appendix ~\ref{app:additional_imp_learning}). Using $T=1$ for Penalty and 1000 validation points, we get an accuracy of 87.41 $\pm$ 0.26 (mean $\pm$ s.d. of 5 trials) comparable to 86.92 (\citeauthor{ren18l2rw}) and 89.27 (\citeauthor{han2018coteaching}). Thus, Penalty achieves comparable performance to these specialized methods designed for the data denoising problem. The enormous size of WRN-28-10, restricted us to use $T=1$ but we expect larger $T$ to improve the results (Fig.~\ref{fig:effect_of_T}(a) in Appendix~\ref{sec:impact_of_T}). 
We also compared Penalty against an RMD-based method [\cite{franceschi2017forward}], using the same setting as their Sec.~5.1, on a subset of MNIST data corrupted with 50\% label noise and softmax regression as the model (see Appendix~\ref{app:additional_imp_learning}). The accuracy of the classifier trained on a subset of the data with points having importance values greater than 0.9 (as computed by Penalty with $T=20$) along with the validation set is 90.77\% better than 90.09\% reported by the RMD-based method.

\begin{table*}[tb]
  \caption{\small{Few-shot classification accuracy (\%) on Omniglot and Mini-ImageNet. We report mean$\pm$s.d. for Omniglot and 95\% confidence intervals for Mini-Imagenet over five trials. Results for learning a common representation using Penalty, ApproxGrad and RMD [\cite{franceschi2018bilevel}] are averaged over 600 randomly sampled tasks from the meta-test set. Results for previous methods using similar models are also reported (MAML [\cite{finn2017model}], iMAML [\cite{rajeswaran2019meta}], Prototypical Networks (Proto Net) [\cite{snell2017prototypical}], Relation Networks (Rel Net) [\cite{sung2018learning}], SNAIL [\cite{mishra2017simple}]}).}
  \label{Table:few-shot_learning}
  \centering
  \resizebox{0.95\textwidth}{!}{
  \small
  \begin{tabular}{c|ccccc|ccc}
    \toprule
    \multirow{2}{*}{} & \multicolumn{5}{c|}{} & \multicolumn{3}{c}{Learning a common representation} \\
    & \makecell{MAML} & \makecell{iMAML} & \makecell{Proto Net} & \makecell{Rel Net} & \makecell{SNAIL} 
    & \makecell{RMD} & ApproxGrad & Penalty \\
    \midrule
    Omniglot & \multicolumn{8}{c}{}\\
    \midrule
    5-way 1-shot & 98.7 & {\bf99.50}$\pm$0.26 & {98.8} & {\bf99.6}$\pm$0.2 & {99.1} & {\bf98.6} & 97.75$\pm$0.06 & 97.83$\pm$0.35 \\
    5-way 5-shot & {\bf99.9}& 99.74 $\pm$ 0.11 & 99.7 & {\bf99.8}$\pm$0.1 & {99.8} & {\bf99.5} & {\bf99.51}$\pm$0.05 & {\bf99.45}$\pm$0.05 \\
    20-way 1-shot & 95.8 & 96.18 $\pm$ 0.36 & 96.0 & {\bf97.6}$\pm$0.2 & {\bf97.6} & {\bf95.5} & 94.69$\pm$0.22 & 94.06$\pm$0.17 \\
    20-way 5-shot & 98.9 & 99.14 $\pm$ 0.10 & 98.9 & 99.1$\pm$0.1 & {\bf99.4} & 98.4 & {\bf98.46}$\pm$0.08 & {\bf98.47}$\pm$0.08 \\
    \midrule
    Mini-Imagenet & \multicolumn{8}{c}{}\\
    \midrule
    5-way 1-shot & 48.70$\pm$1.75 & 49.30 $\pm$ 1.88 & 49.42$\pm$0.78 & 50.44$\pm$0.82 & {\bf55.71}$\pm$0.99 &50.54$\pm$0.85 & 43.74$\pm$1.75 & {\bf53.17}$\pm$0.96 \\
    5-way 5-shot & 63.11$\pm$0.92 & - & 68.20$\pm$0.66 & 65.32$\pm$0.82 & {\bf68.88}$\pm$0.92 & 64.53$\pm$0.68 & 65.56$\pm$0.67 & {\bf67.74}$\pm$0.71 \\
    \bottomrule
  \end{tabular}
  }
\end{table*}

\subsection{Few-shot learning}\label{sec:experiment2}
Next, we evaluate the performance of Penalty on the task of learning a common representation for the few-shot learning problem. We use the formulation presented in Eq.~(\ref{eq:meta-learning}) and use Omniglot [\cite{lake2015human}] and Mini-ImageNet [\cite{vinyals2016matching}] datasets for our experiments. Following the protocol proposed by [\cite{vinyals2016matching}] for $N$-way $K$-shot classification, we generate meta-training and meta-testing datasets. Each meta-set is built using images from disjoint classes. For Omniglot, our meta-training set comprises of images from the first 1200 classes and the remaining 423 classes are used in the meta-testing dataset. We also augment the meta-datasets with three different rotations (90, 180, and 270 degrees) of the images as used by [\cite{santoro2016one}]. For the experiments with Mini-Imagenet, we used the split of 64 classes in meta-training, 16 classes in meta-validation, and 20 classes in meta-testing as used by [\cite{ravi2016optimization}].

Each meta-batch of the meta-training and meta-testing dataset comprises of a number of tasks which is called the meta-batch-size. Each task in the meta-batch consists of a training set with $K$ images and a testing set consists of 15 images from $N$ classes. We train Penalty using a meta-batch-size of 30 for 5 way and 15 for 20-way classification for Omniglot and with a meta-batch-size of 2 for Mini-ImageNet experiments. The training sets of the meta-train-batch are used to train the lower-level problem and the test sets are used as validation sets for the upper-level problem in Eq.~(\ref{eq:meta-learning}). The final accuracy is reported using the meta-test-set, for which we fix the common representation learned during meta-training. We then train the classifiers at the lower-level for 100 steps using the training sets from the meta-test-batch and evaluate the performance of each task on the associated test set from the meta-test-batch. The average performance of Penalty and ApproxGrad over 600 tasks is reported in Table~\ref{Table:few-shot_learning}. Penalty outperforms other bilevel methods namely the ApproxGrad (trained with 20 lower-level iterations and 20 updates for the linear system) and the RMD-based method [\cite{franceschi2018bilevel}] on Mini-Imagenet and is comparable to them on Omniglot. We demonstrate the convergence speed of Penalty in comparison to ApproxGrad (Fig.~\ref{fig:acc_vs_wct}(b)) and the trade-off between using higher $T$ and time for the two methods (see Fig.~\ref{fig:effect_of_T} and Appendix~\ref{sec:impact_of_T} for a detailed evaluation of the impact of T on the performance of Penalty) and show that Penalty converges much faster than ApproxGrad. In comparison to non-bilevel approaches that used models of a similar size as ours, Penalty is comparable to most approaches and is only slightly worse than [\cite{mishra2017simple}] which makes use of temporal convolutions and soft attention.
We used convolutional neural networks and a residual network for learning the common task representation (upper-level) for Omniglot and Mini-ImageNet, respectively, and use logistic regression to learn task-specific classifiers (lower-level). 

\subsection{Training-data poisoning}\label{sec:experiment3_1}
Here, we evaluate Penalty on the clean label data poisoning attack problem. We use the setting presented in [\cite{shafahi2018poison}] which adds a single poison point to misclassify a particular target image from the test set. 
% and use the feature collision mechanism of [\cite{shafahi2018poison} in bilevel optimization to generate poisoning points. 
We use the dog vs. fish dataset and InceptionV3 network as the representation map. We choose a target image $t$ and a base image $b$ from the test set such that the representation of the base is closest to that of the target but has a different label. The poison point is initialized from the base image. Unlike the original approach [\cite{shafahi2018poison}], we use a bilevel formulation along with a constraint on the maximum perturbation. We solve the bilevel problem in Eq.~(\ref{eq:clean-label-poisoning}) with an additional feature collision term ($\|r(t) - r(u)\|_2^2$) from [\cite{shafahi2018poison}] which is shown to be helpful in practice
%\begin{figure}[tp]
\begin{wrapfigure}[12]{r}{0.45\textwidth}
\vspace{-0.4cm}
\small
\begin{center}
\includegraphics[width=0.45\textwidth]{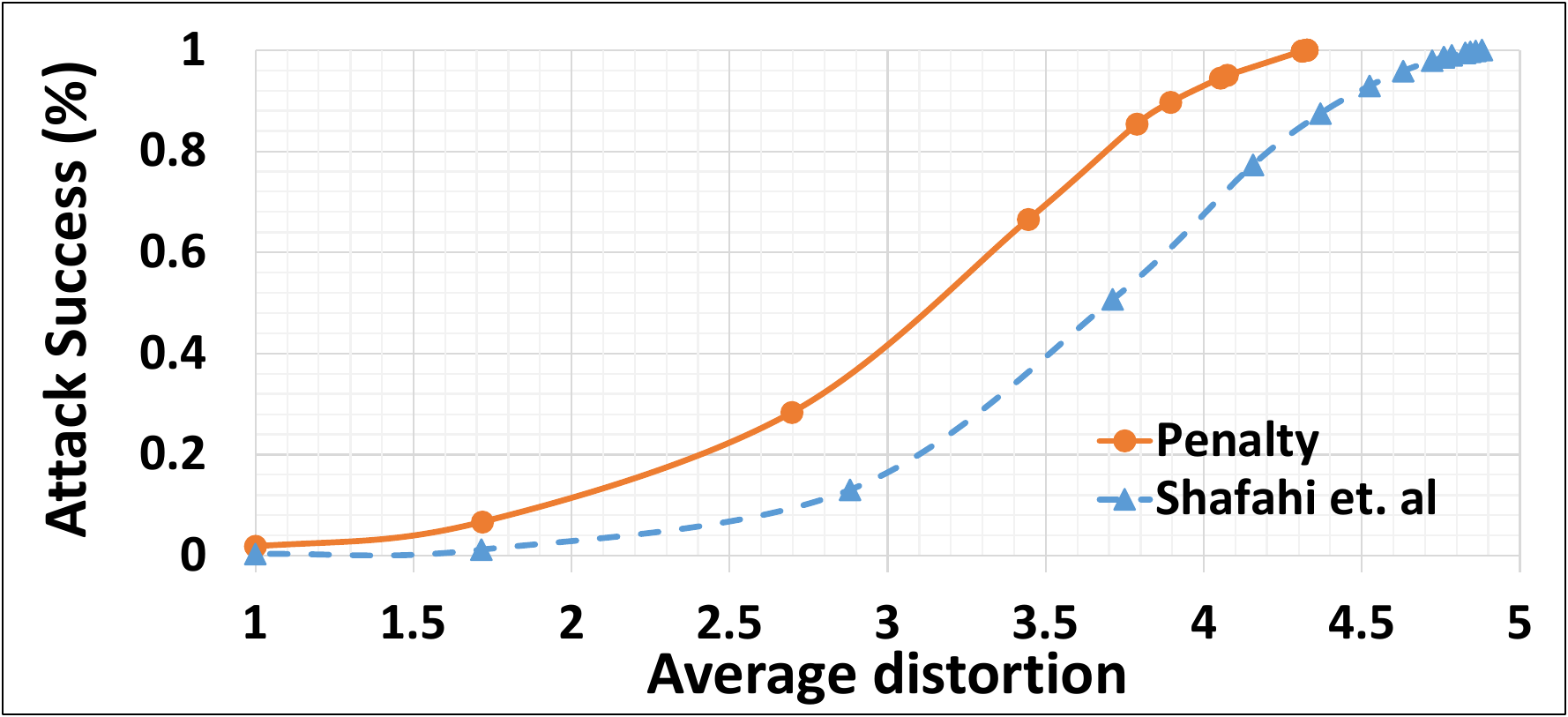}
\vspace{-0.8cm}
\caption{Penalty finds clean label poisoning data with smaller distortion compared to \cite{shafahi2018poison} making the attack stealthier.}
\label{fig:dataset_posioning_dogfish}
\end{center}
\end{wrapfigure}
where $r(\cdot)$ is the 2048-dimensional representation map. The lower-level problem trains a softmax classifier on top of this representation. (The full problem is in Eq.~(\ref{eq:clean-label-poisoning-full}) of Appendix~\ref{app:clean-label-full}).
We evaluate the attack success by retraining the softmax classifier on the clean dataset augmented with the poison point. The attack is considered successful if the target point is misclassified after retraining. 
%The attack success of Penalty using just a single poisoned point per target image is 100\% 
%with $\epsilon = 16$ (maximum distortion)
%. For comparison attack success with using a single poisoned point by [\cite{koh2017understanding} and [\cite{shafahi2018poison} are 57\% and 100\% respectively. 
Choosing each correctly classified point from the test set as the target we search for the smallest $\epsilon \in \{1,2,...,16\}$ that when used as the upper bound for the perturbation of the poison image, leads to misclassification of that target. 
%We compare the average distortion found by Penalty against the approach presented in [\cite{shafahi2018poison}. 
For comparison, we use the Alg. 1 from [\cite{shafahi2018poison}] using $\ell_2$ projection after each step to constrain the amount of perturbation. Similar to Penalty, the smallest $\epsilon $ that causes misclassification is recorded for each target (see Appendix~\ref{app:clean-label-full} for details). Fig.~\ref{fig:dataset_posioning_dogfish} shows that Penalty achieves higher attack success with the same amount of distortion, or in another view, achieves the same attack success with much less distortion. We ascribe this benefit to the use of the bilevel method as opposed to a non-bilevel approach in [\cite{shafahi2018poison}].
Poison points generated by Penalty are shown in Fig.~\ref{fig:attack_images} in Appendix.
Lastly, we also test Penalty on a data poisoning problem without upper-level constraint (Appendix~\ref{sec:experiment3_2}). We compare it with RMD and ApproxGrad and show that Penalty outperforms RMD and is comparable to ApproxGrad. 

\begin{figure*}[t]
\small

\centering
\subfigure[Data denoising for MNIST with 25\% label noise and 1000 validation points. Figures show the performance of the classifier learned using the importance re-weighted dataset during bilevel training. 75\% train accuracy and 100\% validation accuracy is an indication that the bilevel problem is correctly solved. We see that Penalty reaches this point earlier than ApproxGrad signifying better convergence speed. High test accuracy indicates the importance values are enabling good generalization properties on the test set.]
{
\label{fig:data_denoising_wct_1}
\includegraphics[width=\textwidth,height=3.5cm]{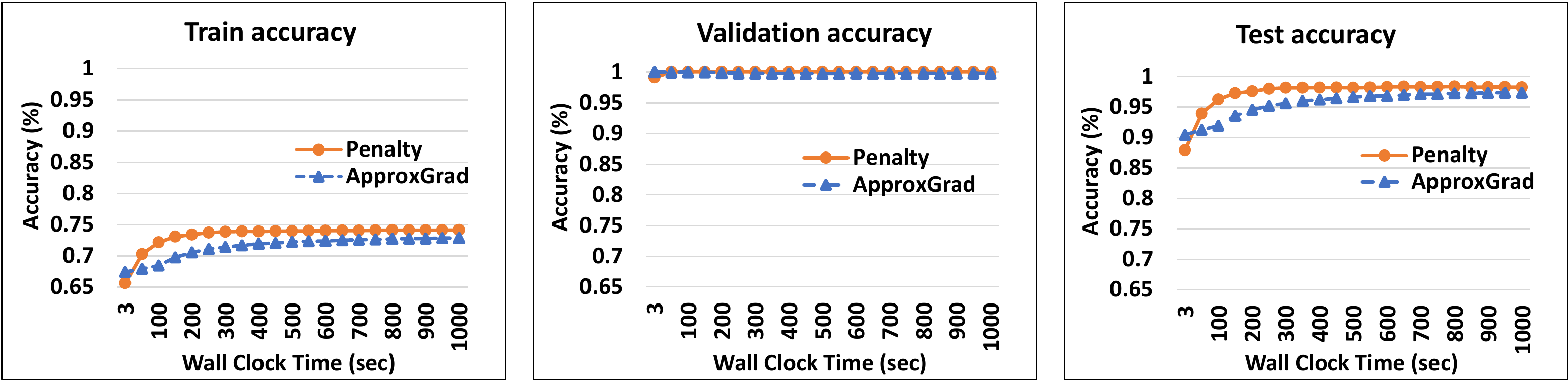}
}
\subfigure[5-way 5-shot learning on Miniimagenet. Figures show the average performance of the softmax classifier trained using the train sets of meta-train and meta-test sets, respectively, on top of the common representation which is learned during bilevel training. The accuracy on the test set of meta-train indicates how well the upper-level problem is being solved. High performance on test sets of meta-test shows that the common representation generalizes to new tasks not observed during training.]{
\label{fig:few_shot_wct_1}
\includegraphics[width=0.8\textwidth,height=3.7cm]{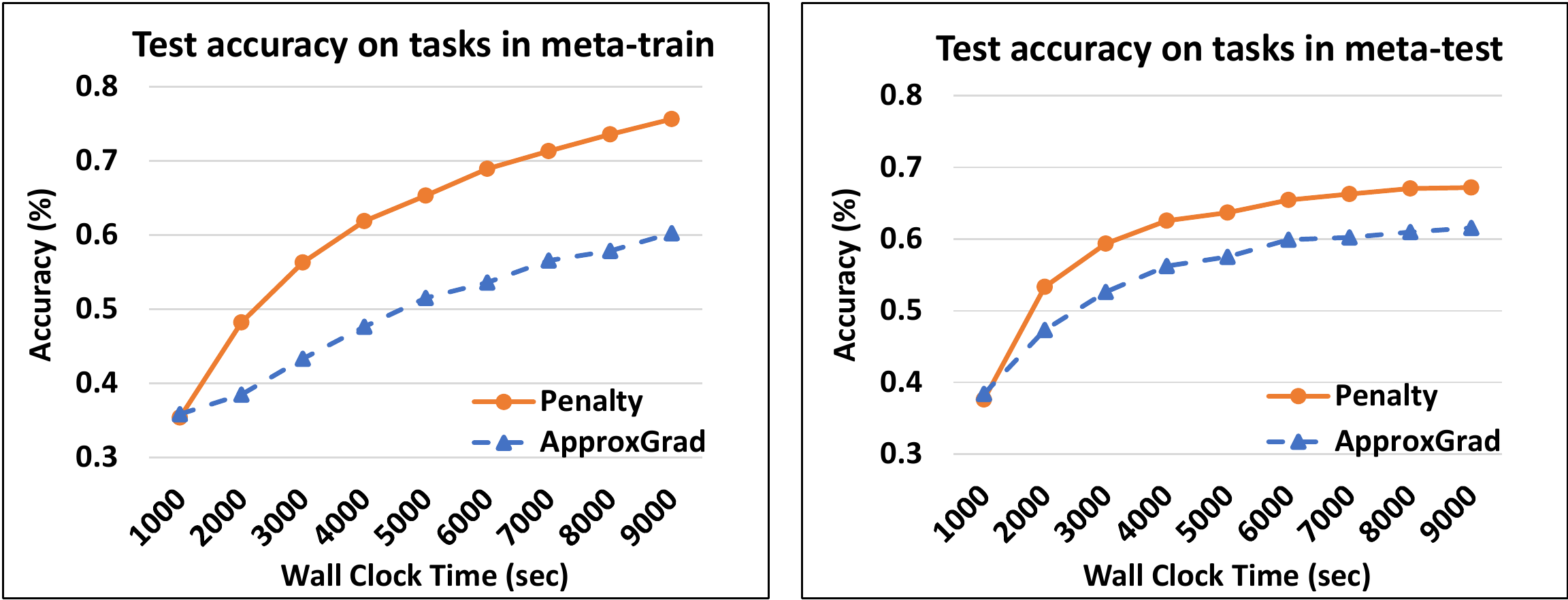}
}
%\subfigure[Untargeted data poisoning]{
%\label{fig:data_posioning_wct_1}\includegraphics[width=0.3\columnwidth,height=4cm]{acc_wct_dp.pdf}
%}
\vspace{-0.3cm}
\caption{
\small{Comparison of accuracy and wall clock time (convergence speed) during bilevel training with Penalty and ApproxGrad on the data denoising problem (Sec.~\ref{sec:experiment1}) and the few-shot learning problem (Sec.~\ref{sec:experiment2}).
}}
\vspace{-0.9cm}
\label{fig:acc_vs_wct}
\end{figure*}
\vspace{-0.3cm}
\subsection{Convergence speed comparison of Penalty and ApproxGrad}
\label{sec:convergence_speed}
Here, we compare the accuracy and wall clock time for Penalty and ApproxGrad on real (Fig.~\ref{fig:acc_vs_wct}) and synthetic (Table~\ref{Table:wall-clock time} in Appendix) problems and show that Penalty converges faster than ApproxGrad. We use the data denoising and few-shot learning problems discussed earlier for this comparison. 
%For untargeted data poisoning the methods have similar convergence properties. 
In Fig.~\ref{fig:acc_vs_wct}, the training accuracy is indicative of the optimality of the lower-level problem and the validation/meta-train-test accuracy shows how well the upper-level problem is being solved during bilevel training. For data denoising, the test accuracy is a measure of how well the classifier learned using the importance values obtained from the bilevel training generalizes on the test set. On the other hand, for few-shot learning, the test accuracy on the tasks in the meta-test set is a measure of how well the common representation, learned from bilevel training, generalizes to new unseen tasks after training the softmax classifier on the training set of the meta-test dataset. 
%For untargeted poisoning the test accuracy shows the degradation in the accuracy of a classifier trained on the poisoned data. 

We see that for the importance learning problem Penalty quickly reaches a training accuracy of roughly 75\% which is desirable since the dataset contains 25\% noise. High test and validation accuracy are indicative that the importance values learned for the training data indeed helps to learn a good classifier. To report the performance for few-shot learning we train the softmax classifier on top of the common representation using the data from the meta-train-train and meta-test-train sets and then evaluate their performance on meta-train-test and meta-test-test sets. The results in Fig.~\ref{fig:acc_vs_wct}(b) are an average of over 600 tasks. The performance on the meta-train-test is indicative of how well the upper-level problem is being solved. Accuracy on meta-test-test indicates the performance of the learned representation on new tasks which have not been observed during training. In both the experiments we see Penalty has a better convergence speed than ApproxGrad. This directly translates to shorter running times for Penalty, making it a significantly more practical algorithm in comparison to existing bilevel solvers used for machine learning problems.
\vspace{-0.3cm}
%%%%%%%%%%%%%%%%%%%%%%%%%%%%%%%%%%%%%%%%%%%%%%%%%%%%%%%%%%%%%%%%%%%%%%%%%%%%%%%%%%%%%%%%%%%%%%%%%%%%%%%%%%%
\section{Conclusion}\label{sec:conclusion}
%%%%%%%%%%%%%%%%%%%%%%%%%%%%%%%%%%%%%%%%%%%%%%%%%%%%%%%%%%%%%%%%%%%%%%%%%%%%%%%%%%%%%%%%%%%%%%%%%%%%%%%
A wide range of important machine learning problems can be expressed as bilevel optimization problems, and new applications are still being discovered. So far, the difficulty of solving bilevel optimization has limited its widespread use for solving large problems involving deep models. In this work, we presented an efficient algorithm based on penalty function which is simple and has practical advantages over existing methods. Compared to previous methods we demonstrated our method's ability to handle constraints, achieve competitive performance on problems with convex lower-level costs, and get significant improvements on problems with non-convex lower-level costs in terms of accuracy and convergence speed, highlighting its effectiveness in deep learning settings. In future works, we plan to tackle other challenges in bilevel optimization such as handling problems with non-unique lower-level solutions.

%%%%%%%%%%%%%%%%%%%%%%%%%%%%%%%%%%%%%%%%%%%%%%%%%%%%%%%%%%%%%%%%%%%%%%%%%%%%%%
\section{Acknowledgement}
%%%%%%%%%%%%%%%%%%%%%%%%%%%%%%%%%%%%%%%%%%%%%%%%%%%%%%%%%%%%%%%%%%%%%%%%%%%%%%
We thank the anonymous reviewers for their insightful comments and suggestions. This work was supported by the NSF EPSCoR-Louisiana Materials Design Alliance (LAMDA) program \#OIA-1946231.

%\bibliographystyle{plain}
%\small
\vspace{-0.3cm}
\bibliography{acml21}

\clearpage
\appendix

%%%%%%%%%%%%%%%%%%%%%%%%%%%%%%%%%%%%%%%%%%%%%%%%%%%%%%%%%%%%%%%%%%%%%%%%%%%%%%%%%%%%%%%%%
%\begin{appendices}
\begin{center}
{\LARGE \bf Appendix}
\end{center}
We provide the missing proofs in Appendix~\ref{app:proofs},
review other methods for solving bilevel problems in Appendix~\ref{app:other methods}.
We discuss the modifications to improve the Alg.~\ref{alg:main} in Appendix~\ref{app:extensions}, show additional experiment results in Appendix~\ref{app:additional_experiments} and conclude by providing experiment details in Appendix~\ref{app:details_experiments}.
%%%%%%%%%%%%%%%%%%%%%%%%%%%%%%%%%%%%%%%%%%%%%%%%%%%%%%%%%%%%%%%%%%%%%%%%%%%%%%%%%%%%%%%%%%%%%%%%%%%%%%%%%%%%%%%%%%%%%
\section{Proofs} \label{app:proofs}
%%%%%%%%%%%%%%%%%%%%%%%%%%%%%%%%%%%%%%%%%%%%%%%%%%%%%%%%%%%%%%%%%%%%%%%%%%%%%%%%%%%%%%%%%%%%%%%%%%%%%%%%%%%%%%%%%%%%%
\begin{theoremrep}[\ref{thm:main}]
Suppose $\{\epsilon_k\}$ is a positive ($\epsilon_k>0$) and convergent ($\epsilon_k \to 0$) sequence,
$\{\gamma_k\}$ is a positive ($\gamma_k>0$), non-decreasing ($\gamma_1 \leq \gamma_2 \leq \cdots$),
and divergent ($\gamma_k \to \infty$) sequence.
Let $\{(u_k,v_k)\}$ be the sequence of approximate solutions to Eq.~(\ref{eq:bilevel_with_constraints_penalty_formulation}) with tolerance 
$(\nabla_u \tilde{f}(u_k,v_k))^2 + (\nabla_v \tilde{f}(u_k,v_k))^2 \leq \epsilon_k^2$ for all $k=0,1,\cdots$ and LICQ is satisfied at the optimum. 
Then any limit point of $\{(u_k,v_k)\}$ satisfies the KKT conditions of the problem in Eq.~(\ref{eq:bilevel_with_constraints_single_reformulation}).
\end{theoremrep}

\begin{proof}
The proof follows the standard proof for penalty function methods, e.g., [\cite{nocedal2006numerical}].
Let $w:=(u,v)$ refer to the pair, and let $\wo:=(\uo,\vo)$ be any limit point of the sequence $\{w_k:=(u_k,v_k)\}$, and 
\[
\tilde{g}:=\left(
\begin{array}{c}h(u,v)\\ \gv(u,v)\end{array}\right)
\]
 
then there is a subsequence $\mathcal{K}$ such that $\lim_{k \in \mathcal{K}} w_k = \wo$.
From the tolerance condition
\begin{eqnarray*}
\|\nabla_w \tilde{f}(w_k;\gamma_k)\| = \;
\|\nabla_w f(w_k) + \gamma_{k} J_{w}^T(\tilde{g}(w_k)) \tilde{g}(w_k)\| \leq \epsilon_k
\end{eqnarray*}
we have
\begin{eqnarray*}
\|J_{w}^T (\tilde{g}(w_k)) \tilde{g}(w_k)\|  \leq \frac{1}{\gamma_k}[\|\nabla_w f(w_k)\| + \epsilon_k]
\end{eqnarray*}
Take the limit with respect to the subsequence $\mathcal{K}$ on both sides to get
\(J_{w}^T (\tilde{g}(\wo)) \tilde{g}(\wo) = 0\).
Assuming linear independence constraint quantification (LICQ) we have that the columns of the $J_{w}^T\tilde{g} = \left(\begin{array}{cc}J_{u}^T h & \guv\\ J_v^T h &\gvv\end{array}\right)
%=\left(\begin{array}{cc}J_{u}^T h & \guv\\ 0 &\gvv\end{array}\right)
$ are linearly independent. 
%is a tall matrix and $\gvv$ is invertible by assumption, $\nabla_{wv}^2 g$ is full-rank and 
Therefore $\tilde{g}(\wo) = 0$, which is the primary feasibility condition for 
Eq.~(\ref{eq:bilevel_with_constraints_single_reformulation}).
Furthermore, let $\mu_k := -\gamma_k \tilde{g}(w_k)$, then by definition, 
\begin{eqnarray*}
\nabla_w \tilde{f}(w_k;\gamma_k) = \nabla_w f(w_k) - J_{w}^T(\tilde{g}(w_k)) \mu_k
\end{eqnarray*}
We can write 
\begin{eqnarray*}
\left[J_{w}(\tilde{g}(w_k)) J_{w}^T(\tilde{g}(w_k))\right]\mu_k = 
J_{w}(\tilde{g}(w_k)) \left[\nabla_w f(w_k)-\nabla_w \tilde{f}(w_k;\gamma_k) \right]
\end{eqnarray*}

The corresponding limit $\overline{\mu}$ can be found by taking the limit of the subsequence $\mathcal{K}$
\begin{eqnarray*}
\overline{\mu}:=\lim_{k\in\mathcal{K}} \mu_k = 
\left[J_{w}(\tilde{g}(\wo)) J_{w}^T(\tilde{g}(\wo))\right]^{-1}J_{w}(\tilde{g}(w_k))\nabla_w f(\wo)
\end{eqnarray*}
Since $\lim_{k\in\mathcal{K}} \nabla_w \tilde{f}(w_k;\gamma_k) = 0$ from the condition $\epsilon_k \to 0$,
we get 
\begin{eqnarray*}
\nabla_w f(\wo) - J_{w}^T(\tilde{g}(\wo)) \overline{\mu} = 0   
\end{eqnarray*}
at the limit $\wo$,
which is the stationarity condition of Eq.~(\ref{eq:bilevel_with_constraints_single_reformulation}). 
Together with the feasibility condition $\tilde{g}(\wo) = 0$, the two KKT conditions of
Eq.~(\ref{eq:bilevel_with_constraints_single_reformulation}) are satisfied at the limit point. 
\end{proof}

\if0
\begin{corollaryrep}[\ref{cor:main}]
%In addition to the conditions of Theorem~\ref{thm:main}, if $f$ and $g$ are convex functions
%of $v$ for any $u$, then any limit point of Theorem~\ref{thm:main} is an optimal solution $(\us,\vs)$ of  Eq.~(\ref{eq:bilevel}). {\bf ?????}
%n optimal solution $(\us,\vs)$ of  Eq.~(\ref{eq:bilevel}). 
In addition to the conditions of Theorem~\ref{thm:main}, if $g$ is convex in $v$ for any $u$, then any limit point of Theorem~\ref{thm:main} also satisfies the KKT conditions of Eq.~(\ref{eq:bilevel}). 
Furthermore if $f$ is convex in $(u,v)$ then the limit point of Theorem~\ref{thm:main} is an optimal solution $(\us,\vs)$ of Eq.~(\ref{eq:bilevel}).
\end{corollaryrep}
%Proof is omitted.
\begin{proof}
From Theorem~\ref{thm:main}, the limit point satisfies the KKT conditions of
\end{proof}
%{\it Comment.}
%The requirement of $f(u,\vs(u))$ being convex in $u$ is weaker than $f(u,v)$ being convex in $(u,v)$. 
%This is a consequence of the fact that the convexity of $f$ and $g$ implies that the KKT
%condition for Eq.~(\ref{eq:bilevel_necessary}) is also sufficient, and that
%the solution for Eq.~(\ref{eq:bilevel_necessary}) is also the solution for %Eq.~(\ref{eq:bilevel}).
%\end{proof}
\fi

\begin{lemmarep}[\ref{lm:main}]
Assume $h \equiv 0$. \\ Given $u$, let $\vh$ be $\vh:=\arg\min_v \tilde{f}(u,v;\gamma)$ from Eq.~(\ref{eq:bilevel_with_constraints_penalty_formulation}).
Then, $\nabla_u\tilde{f}(u,\vh;\gamma)=\frac{df}{du} (u,\vh)$ as in Eq.~(\ref{eq:hypergradient}).
\end{lemmarep}
\begin{proof}
At the minimum $\hat{v}$ the gradient $\nabla_v \tilde{f}$ vanishes, that is 
$\fv + \gamma \gvv\gv = 0$. \\Equivalently, $\gv = -\gamma^{-1} (\gvv)^{-1} \fv$.
Then,
\begin{eqnarray*}
\nabla_u \tilde{f}(\vh) = \fu(\vh) + \gamma \guv(\vh)\gv(\vh) 
= \fu(\vh) - \guv(\vh)\gvv^{-1}(\vh)\fv(\vh),
\end{eqnarray*}
where $\gamma$ disappears, which is the hypergradient $\frac{df}{du}(u,\vh)$ as in Eq.~(\ref{eq:hypergradient}).
\end{proof}
That is, if we find the minimum $\vh$ of the penalty function for given $u$ and $\gamma$, 
we get the hypergradient Eq.~(\ref{eq:hypergradient}) at $(u,\vh)$. 
Furthermore, under the conditions of Theorem~\ref{thm:strict},
$\vh(u) \to \vs(u)$ as $\gamma \to \infty$ (see Lemma 8.3.1 of [\cite{bard2013practical}]), and we get the exact hypergradient asymptotically.

\section{Review of other bilevel optimization methods for unconstrained problems}\label{app:other methods}
%%%%%%%%%%%%%%%%%%%%%%%%%%%%%%%%%%%%%%%%%%%%%%%%%%%%%%%%%%%%%%%%%%%%%%%%%%%%%%%%%%%%%%%%%%
Several methods have been proposed to solve bilevel optimization problems appearing in machine learning,
including forward/reverse-mode differentiation  [\cite{maclaurin2015gradient,franceschi2017forward}]
and approximate gradient [\cite{domke2012generic,pedregosa2016hyperparameter}] described briefly here.

{\bf Forward-mode (FMD) and Reverse-mode differentiation (RMD).}
Domke [\cite{domke2012generic}], Maclaurin et al.[\cite{maclaurin2015gradient}],  
Franceschi et al.~[\cite{franceschi2017forward}], and Shaban et al.~[\cite{shaban2018truncated}]
studied forward and reverse-mode differentiation to solve the
minimization problem $\min_u f(u,v)$ where the lower-level variable $v$ 
follows a dynamical system $v_{t+1} =\Phi_{t+1}(v_{t};u),\;\;t=0,1,2,\cdots,T-1$.
This setting is more general than that of a bilevel problem. 
However, a stable dynamical system is one that converges to a steady state and
thus, the process $\Phi_{t+1}(\cdot)$ can be considered as minimizing an energy or a potential function.  

Define $A_{t+1}:=\nabla_v \Phi_{t+1}(v_t)$ and $B_{t+1}:=\nabla_u \Phi_{t+1}(v_t)$,
then the hypergradient Eq.~(\ref{eq:hypergradient}) can be computed by
\begin{eqnarray*}
%\small
\frac{df}{du} = \nabla_u f(u,v_T) + \sum_{t=0}^T B_t A_{t+1} \times\cdots\times A_T \nabla_v f(u,v_T)
\end{eqnarray*}

When the lower-level process is one step of gradient descent on a cost function $g$, that is, 
\begin{eqnarray*}
\Phi_{t+1}(v_t;u) = v_t - \rho \gv(u,v_t)
\end{eqnarray*}
we get 
\[
{A_{t+1}=I-\rho \gvv(u,v_t)},\;\;{B_{t+1}=-\rho\guv(u,v_t)}.
\]

$A_t$ is of dimension $V \times V$ and $B_t$ is of dimension $V \times U$. The sequences $\{A_{t}\}$ and $\{B_{t}\}$ can be computed in forward or reverse mode.
\\\\
For {\bf reverse-mode differentiation}, first compute
\[
{v_{t+1}=\Phi_{t+1}(v_t),\;\; t=0,1,\cdots,T\textrm{-}1},
\]
then compute 
\begin{eqnarray*}
&&q_T \leftarrow \nabla_{v} f(u,v_T),\;p_T \leftarrow \nabla_u f(u,v_T)\\
&&p_{t-1} \leftarrow p_t + B_t q_t,\;q_{t-1}\leftarrow A_t q_t,\;\;t=T,T\textrm{-}1,\cdots,1.
\end{eqnarray*}
Time and space Complexity for computing $p_t$ is $O(c)$ since the Jacobian vector product can be computed in $O(c)$ time and space.
The final hypergradient for RMD is $\frac{df}{du} = p_{0}$. \\
Hence the final time complexity for RMD is $O(cT)$ and space complexity is $O(U + VT)$.
\\\\
For {\bf forward-mode differentiation}, 
simultaneously compute $v_t$, $A_{t}$, $B_{t}$ and
\begin{eqnarray*}
P_0 \leftarrow 0,\;\;\;P_{t+1}\leftarrow P_t A_{t+1} + B_{t+1},\;\;\;\;
t=0,1,\cdots,T\textrm{-}1. 
\end{eqnarray*}
Time complexity for computing $P_t$ is $O(Uc)$ since $P_t A_{t+1}$ can be computed using $U$ Hessian vector products each needing $O(c)$ and $B_{t+1}$ also needs $O(Uc)$ time for unit vectors $e_i$ for $i=1...U$. The space complexity for each $P_t$ is $O(UV)$.
The final hypergradient for FMD is 
\[
{\frac{df}{du}=\nabla_u f(u,v_T) + P_T \nabla_v f(v_T)}. 
\]
Hence the final time complexity for FMD is $O(cUT)$ and space complexity is $O(U + UV) = O(UV)$.
\\\\
{\bf Approximate hypergradient (ApproxGrad).} 
Since computing the inverse of the Hessian $(\gvv)^{-1}$ directly is difficult even for moderately-sized
neural networks, Domke [\cite{domke2012generic}] proposed to find an approximate solution to 
$q = (\gvv)^{-1} \fv$ by solving the linear system of equations $\gvv\cdot q \approx \fv$. 
This can be done by solving 
\begin{eqnarray*}
\min_q\;\|\gvv\cdot q -\fv\|
\end{eqnarray*}
using gradient descent, conjugate gradient descent, or any other iterative solver. 
Note that the minimization requires evaluation of the Hessian-vector product,
which can be done in linear time [\cite{pearlmutter1994fast}]. Hence the time complexity of the method is $O(cT)$ and space complexity is $O(U + V)$ since we only need to store a single copy of $u$ and $v$ same as Penalty.
The asymptotic convergence with approximate solutions was shown by [\cite{pedregosa2016hyperparameter}]. In our experiments we use $T$ steps to solve the linear system.

%%%%%%%%%%%%%%%%%%%%%%%%%%%%%%%%%%%%%%%%%%%%%%%%%%%%%%%%%%%%%%%%%%%%%%%%%%%%%%%%%%%%%%%%%%%%%%%%%%%%%%%%
\section{Improvements to Algorithm~\ref{alg:main}}\label{app:extensions}
%%%%%%%%%%%%%%%%%%%%%%%%%%%%%%%%%%%%%%%%%%%%%%%%%%%%%%%%%%%%%%%%%%%%%%%%%%%%%%%%%%%%%%%%%%%%%%%%%%%%%%%%%%
Here we discuss the details of the modifications to Alg.~\ref{alg:main} presented in the main text which can be added to improve the performance of the algorithm in practice. 

\subsection{Improving local convexity by regularization}\label{sec:regularization}
One of the common assumptions of this and previous works is that $\gvv$ is invertible and locally positive definite. Neither invertibility nor positive definiteness hold in general for bilevel 
problems, involving deep neural networks, and this causes difficulties in the optimization.
Note that if $g$ is non-convex in $v$, minimizing the penalty term $\|\gv\|$ does not necessarily lower the cost
$g$ but instead just moves the variable towards a stationary point -- which is a known problem 
even for Newton's method. 
Thus we propose the following modification to the $v$-update:
\begin{eqnarray*}
\min_{v}\;\; \left[\tilde{f} + \lambda_k g\right]
\end{eqnarray*}
keeping the $u$-update intact. 
To see how this affects the optimization, note that $v$-update becomes
\begin{eqnarray*}
v \leftarrow v - \rho \left[ \fv + \gamma_k \gvv \gv + \lambda_k \gv \right]
\end{eqnarray*}
After $v$ converges to a stationary point, we get $\gv = -(\gamma_k \gvv + \lambda_k I)^{-1} \fv$, 
and after plugging this into $u$-update, we get
\begin{eqnarray*}
u\leftarrow u-\sigma\left[\fu-\guv\left(\gvv+\frac{\lambda_k}{\gamma_k}I\right)^{-1}\fv\right]
\end{eqnarray*}
that is, the Hessian inverse $\gvv^{-1}$ is replaced by a regularized version 
$(\gvv + \frac{\lambda_k}{\gamma_k} I)^{-1}$ to improve the positive definiteness of the Hessian. 
With a decreasing or constant sequence $\{\lambda_k\}$ such that $\lambda_k/\gamma_k \to 0$
the regularization does not change to solution. 

\subsection{Convergence with finite $\gamma_k$}\label{sec:augmented lagrangian}
The penalty function method is intuitive and easy to implement, but 
the sequence $\{(\uh_k,\vh_k)\}$ is guaranteed to converge to an optimal solution %$(\us,\vs)$
only in the limit with $\gamma\to\infty$, which may not be achieved in practice in a limited time.
It is known that the penalty method can be improved by introducing an additional term
into the function,
which is called the augmented Lagrangian (penalty) method [\cite{bertsekas1976penalty}]:
\begin{eqnarray*}\label{eq:penalty_aug}
\min_{u,v}\;\;\left[ \tilde{f} + \gv^T \nu \right].
\end{eqnarray*}
This new term $\gv^T \nu$ allows convergence to the optimal solution $(\us,\vs)$ even
when $\gamma_k$ is finite.
Furthermore, using the update rule $\nu \leftarrow \nu + \gamma \gv$, called the method of
multipliers, it is known that $\nu$ converges to the true 
Lagrange multiplier of this problem corresponding to the equality constraints $\gv=0$.
%\textcolor{red}{Empirically, the improvement due to the Lagrangian term seems small, due to the fact that we often do not run the optimization long enough to let $\gamma$ grow sufficiently (e.g., $>10^6$) when solving a large problem.}

\subsection{Non-unique lower-level solution}\label{sec:non-uniqueness}
Most existing methods have assumed that the lower-level solution $\arg\min_v g(u,v)$ is unique
for all $u$. Regularization from the previous section can improve the ill-conditioning of the Hessian $\gvv$ but it does not address the case of multiple disconnected global minima of $g$.
With multiple lower-level solutions $Z(u)=\{v\;|\;v=\arg\min g(u,v)\}$, 
there is an ambiguity in defining the upper-level problem. 
If we assume that $v\in Z(u)$ is chosen adversarially (or pessimistically), 
then the upper-level problem should be defined as 
\begin{eqnarray*}
\min_u \max_{v\in Z(u)} f(u,v). 
\end{eqnarray*}
If $v\in Z(u) $ is chosen cooperatively (or optimistically), then the upper-level problem
should be defined as
\begin{eqnarray*}
\min_u \min_{v\in Z(u)} f(u,v),
\end{eqnarray*}
and the results can be quite different between these two cases.
Note that the proposed penalty function method is naturally solving the optimistic case,
as Alg.~\ref{alg:main} is solving the problem of $\min_{u,v} \tilde{f}(u,v)$ by alternating 
gradient descent. 
However, with a gradient-based method, we cannot hope to find all disconnected multiple solutions. 
In a related problem of min-max optimization, which is a special case of bilevel optimization,
an algorithm for handling non-unique solutions was proposed recently [\cite{hamm2018k}].
This idea of keeping multiple candidate solutions may be applicable to bilevel problems too and further analysis of the non-unique lower-level problem is left as future work. 

\subsection{Modified algorithm}\label{app:mod_alg}
Here we present the modified algorithm which incorporates
regularization (Appendix.~\ref{sec:regularization}) and augmented Lagrangian (Appendix.~\ref{sec:augmented lagrangian}) as discussed previously.
The augmented Lagrangian term $\gv^T \nu$ applies to both $u$- and $v$-update,
but the regularization term $\lambda g$ applies to only the $v$-update 
as its purpose is to improve the ill-conditioning of $\gvv$ during $v$-update.
The modified penalized functions $\tilde{f}_\mathrm{1}$ for $u$-update and 
$\tilde{f}_\mathrm{2}$ for $v$-update are 
\begin{eqnarray*}
\tilde{f}_\mathrm{1}(u,v;\gamma,\nu) := \tilde{f} + \gv^T \nu, \\
\tilde{f}_\mathrm{2}(u,v;\gamma,\lambda,\nu) := \tilde{f} + \gv^T \nu + \lambda g
\end{eqnarray*}

The new algorithm (Alg.~\ref{alg:extended}) is similar to Alg.~\ref{alg:main}
with additional steps for updating $\lambda_k$ and $\nu_k$.
\begin{algorithm}[tb] \caption{Modified Alg.~\ref{alg:main} with regularization and augmented Lagrangian}  \label{alg:extended}
{Input}: $K, T, \{\sigma_k\},\{\rho_{k,t}\}, \gamma_0,\epsilon_0, \lambda_0, \nu_0,  c_\gamma\textrm{(=1.1)}, \\ c_\epsilon\textrm{(=0.9)}, c_\lambda\textrm{(=0.9)}$ \\
{Output}: $(u_{K},v_{T})$\\
Initialize $u_0,v_0$ randomly\\
{Begin}
\begin{algorithmic}
\FOR{$k=0,\;\cdots\;,K\textrm{-}1$}
	\WHILE{$\|\nabla_u \tilde{f}_\mathrm{1}\|^2+\|\nabla_v \tilde{f}_\mathrm{2}\|^2 > \epsilon_{k}^2$}
    	\FOR{$t=0,\cdots,T\textrm{-}1$}
            \STATE{$v_{t+1} \leftarrow v_{t} - \rho_{k,t} \nabla_v \tilde{f}_\mathrm{2}(u_k,v_t)\;$}(Appendix~\ref{app:mod_alg})
        \ENDFOR
        \STATE{$u_{k+1} \leftarrow u_{k} - \sigma_{k} \nabla_u \tilde{f}_\mathrm{1}(u_k,v_T)\;$}(Appendix~\ref{app:mod_alg})
    \ENDWHILE
    \STATE{$\gamma_{k+1} \leftarrow c_\gamma \gamma_{k}$}
    \STATE{$\epsilon_{k+1} \leftarrow c_\epsilon \epsilon_{k}$}
    \STATE{$\lambda_{k+1}\leftarrow c_\lambda \lambda_k$}
    \STATE{$\nu_{k+1} \leftarrow \nu_k + \gamma_k \gv$}
\ENDFOR
\end{algorithmic}
\end{algorithm}

\begin{figure}[t]
\small
\centering
\subfigure[Importance learning]
{
\label{fig:data_denoising_T_1}\includegraphics[width=0.48\columnwidth,height=3cm]{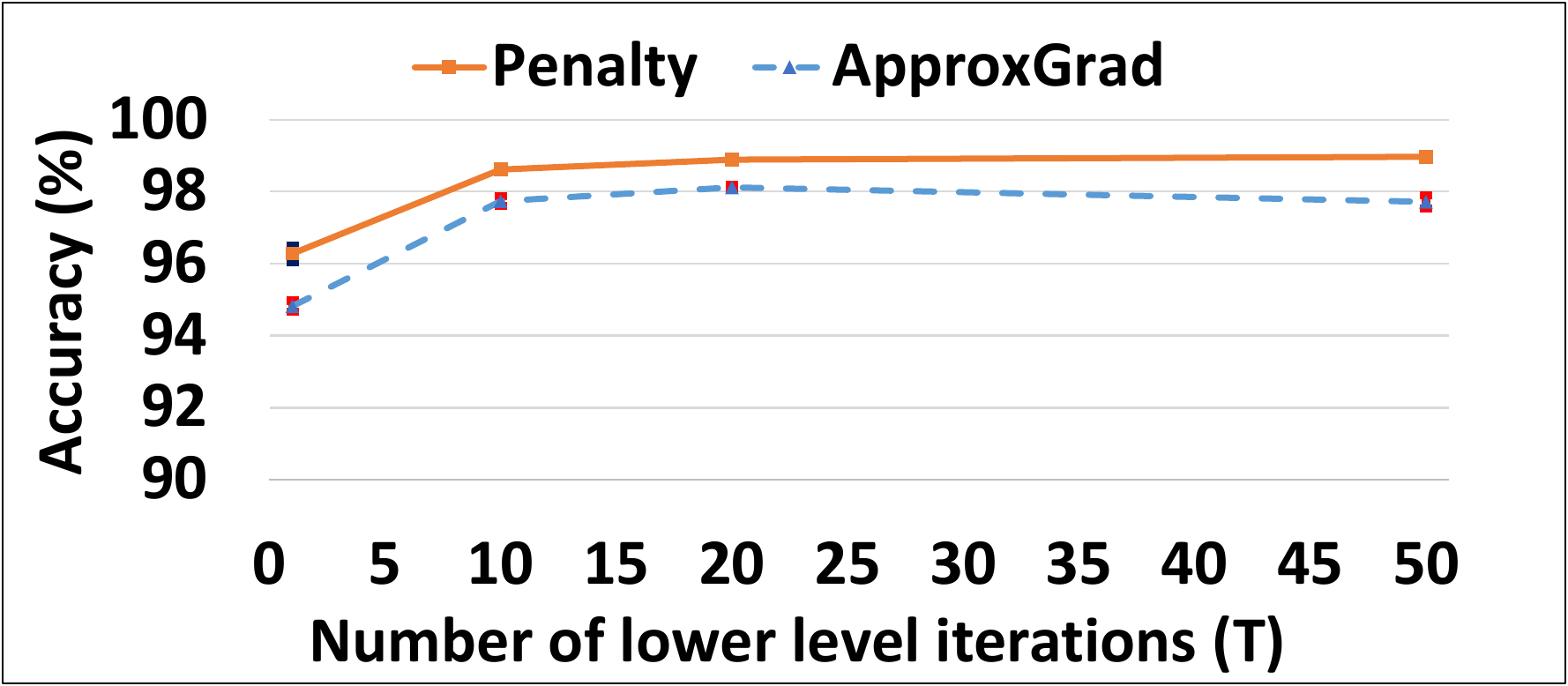}
\label{fig:data_denoising_T_2}\includegraphics[width=0.48\columnwidth,height=3cm]{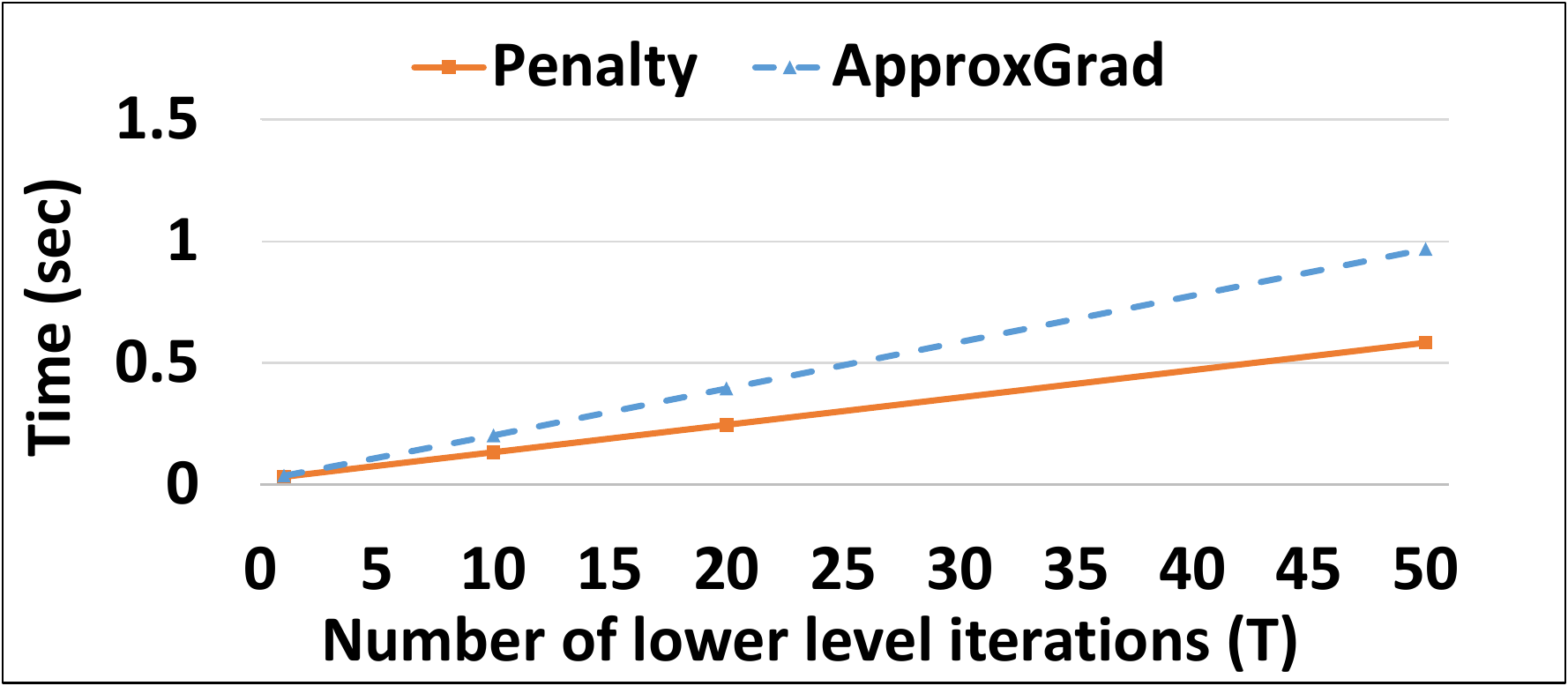}
}
\subfigure[Few-shot learning]{
\label{fig:few_shot_T_1}\includegraphics[width=0.48\columnwidth,height=3cm]{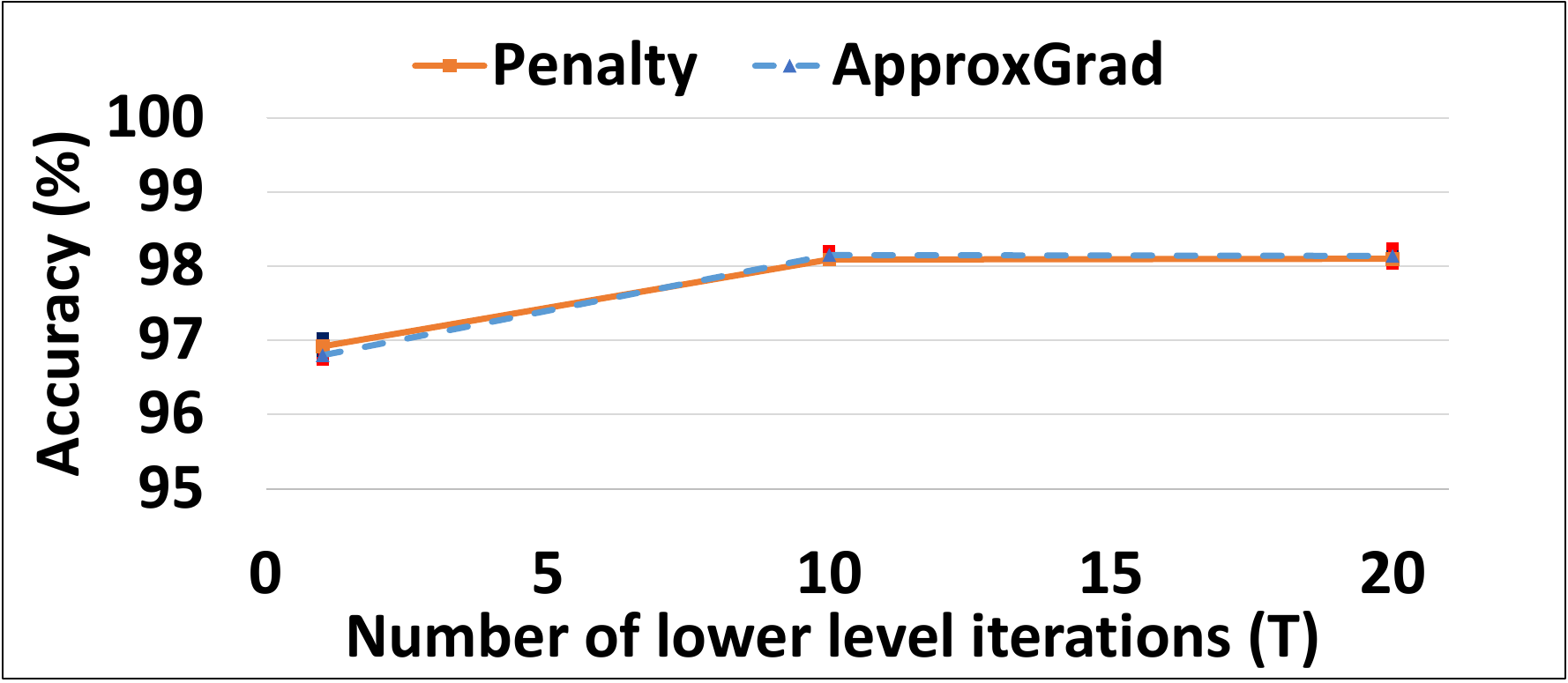}
\label{fig:few_shot_T_2}\includegraphics[width=0.48\columnwidth,height=3cm]{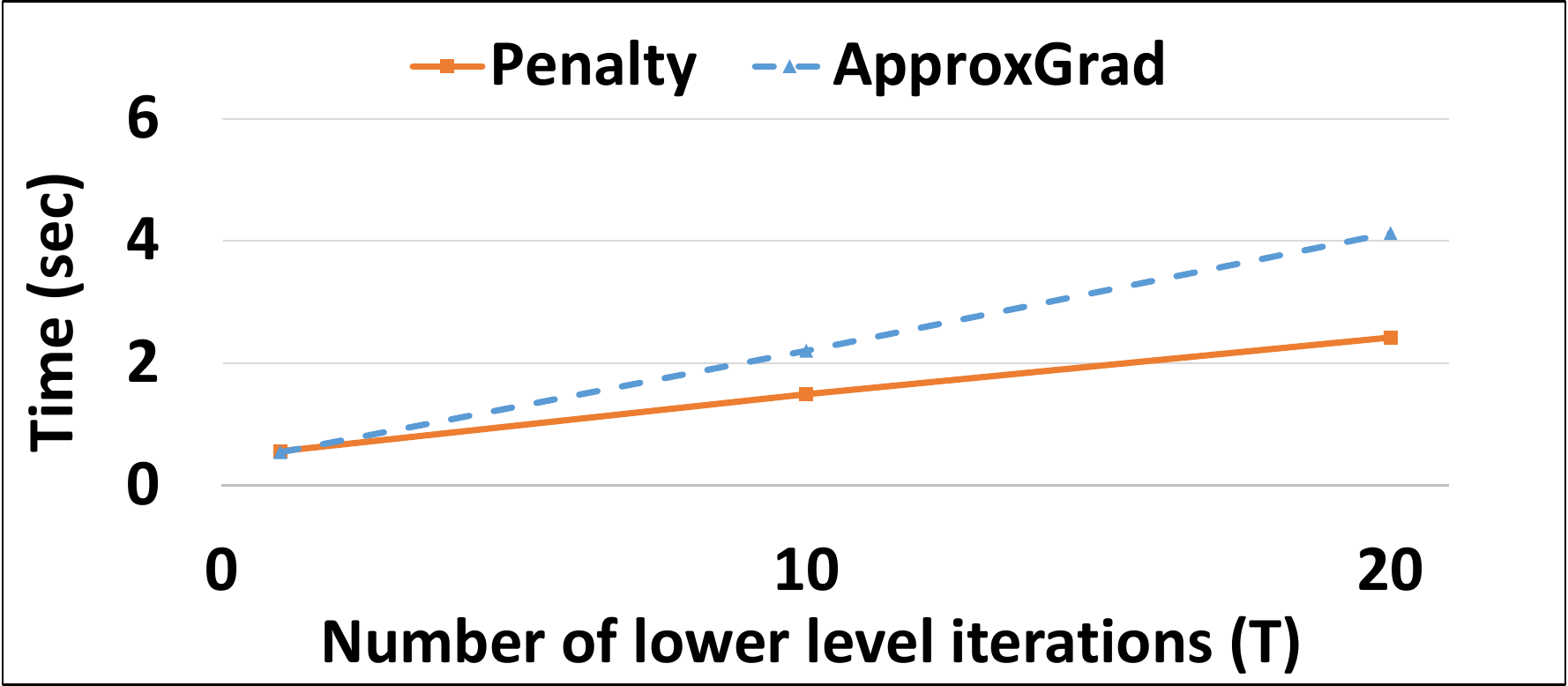}
}
\subfigure[Untargeted data poisoning]{
\label{fig:data_posioning_T_1}\includegraphics[width=0.48\columnwidth,height=3cm]{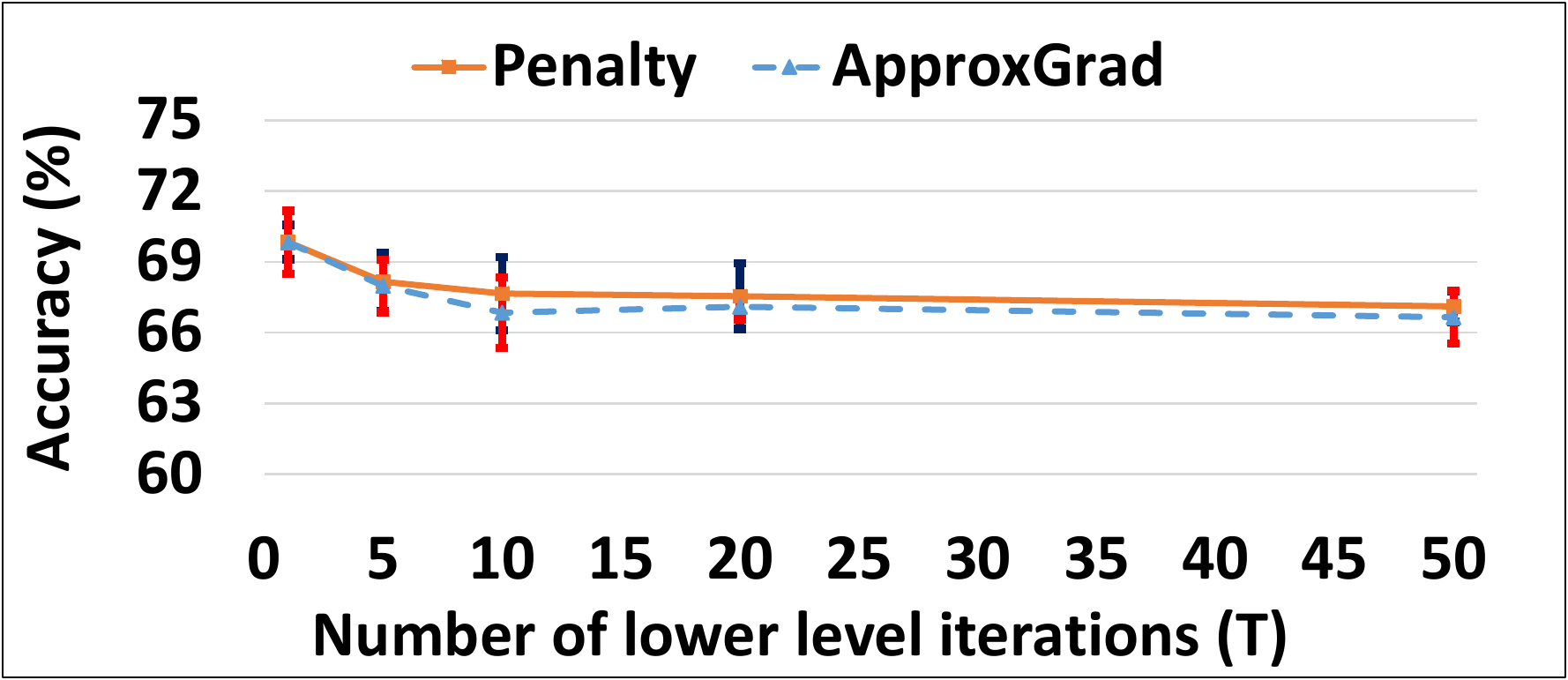}
\label{fig:data_posioning_T_2}\includegraphics[width=0.48\columnwidth,height=3cm]{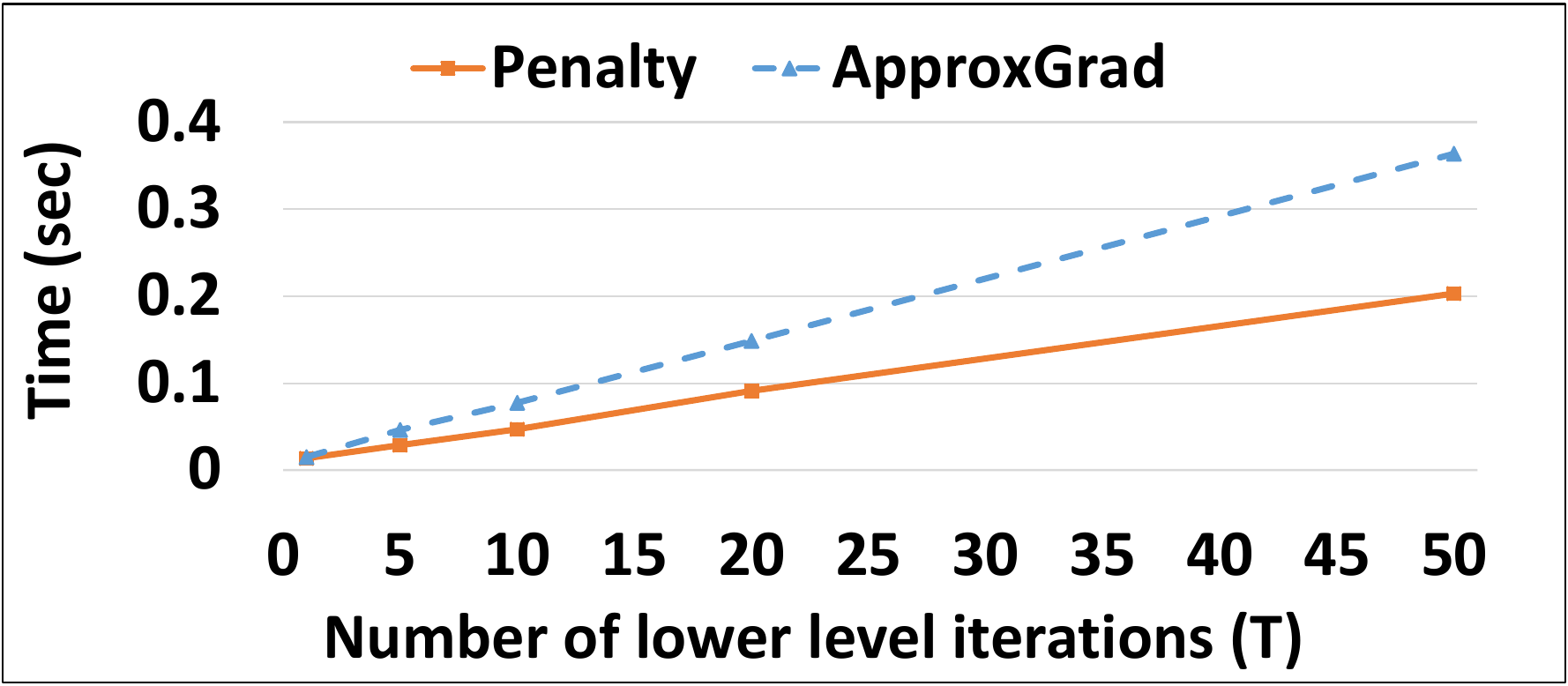}
}
\caption{
\small{Comparison of the final accuracy for different number of lower-level iterations $T$ and wall-clock time required for single upper-level iteration for different values of $T$ for Penalty and ApproxGrad (with $T$ updates for the linear system) on data denoising problem (Sec.~\ref{sec:experiment1} with 25\% noise on MNIST) and few-shot learning problem (Sec.~\ref{sec:experiment2} with 20 way 5 shot classification on Omniglot) and untargeted data poisoning (Appendix.~\ref{sec:experiment3_2} with 60 poisoned points on MNIST)
.}}
\label{fig:effect_of_T}
\end{figure}

\label{sec:Wall-clock-times}
\begin{table}%[tb]
\caption{Mean wall-clock time (sec) for 10,000 upper-level iterations for synthetic experiments. Boldface is the smallest among RMD, ApproxGrad, and Penalty. (Mean $\pm$ s.d. of 10 runs)}
  \label{Table:wall-clock time}
  \centering
   \resizebox{0.6\textwidth}{!}{
   \small
  \begin{tabular}{c|cccc}
    \toprule
      Example 1 & GD & RMD & ApproxGrad & Penalty  \\
    \midrule
    
    T=1 &
    {7.4$\pm$0.3} & {\bf 15.0$\pm$0.1} & {17.4$\pm$0.2} & {17.2$\pm$0.1} \\
    T=5 &
    {14.3$\pm$0.1} & {51.4$\pm$0.3} & {39.3$\pm$2.3} & {\bf 34.3$\pm$0.3}  \\
    T=10 &
    {23.2$\pm$0.1} & {95.4$\pm$0.2} & {60.9$\pm$0.3} & {\bf 57.0$\pm$1.0}  \\
    
    \toprule
    Example 2  &  GD & RMD & ApproxGrad & Penalty \\
    \midrule
    
    T=1 &
     {7.7$\pm$0.1} & {18.5$\pm$0.1} & {\bf 17.2$\pm$0.3} & {17.4$\pm$0.2} \\
    T=5 &
     {17.3$\pm$0.1} & {62.7$\pm$0.1} & {37.9$\pm$0.1} & {\bf 35.0$\pm$0.2} \\
    T=10 &
     {22.4$\pm$2.6} & {115.0$\pm$0.4} & {64.2$\pm$0.3} & {\bf 52.7$\pm$1.4} \\
    
    \toprule
    Example 3  & GD & RMD & ApproxGrad & Penalty \\
    \midrule
    T=1 &
    {8.2$\pm$0.2} & {\bf 18.8$\pm$0.1} & {19.8$\pm$0.1} & {19.1$\pm$0.1} \\
    T=5 &
    {17.4$\pm$0.1} & {72.4$\pm$0.1} & {47.1$\pm$0.4} & {\bf 38.6$\pm$0.4} \\
    T=10 &
    {28.7$\pm$0.6} & {125.0$\pm$9.3} & {80.6$\pm$0.3} & {\bf 62.7$\pm$0.1} \\
    
    \toprule
    Example 4  & GD & RMD & ApproxGrad & Penalty \\
    \midrule
    T=1 &
    {7.9$\pm$0.1} & {\bf 19.5$\pm$0.1} & {20.4$\pm$0.0} & {19.6$\pm$0.1} \\
    T=5 &
    {16.9$\pm$0.2} & {72.8$\pm$0.5} & {48.4$\pm$0.6} & {\bf 40.2$\pm$0.1} \\
    T=10 &
    {28.3$\pm$0.2} & {138.0$\pm$0.2} & {81.2$\pm$1.6} & {\bf 58.0$\pm$4.3} \\
    
    \bottomrule
\end{tabular}
  }
\end{table}

%%%%%%%%%%%%%%%%%%%%%%%%%%%%%%%%%%%%%%%%%%%%%%%%%%%%%%%%%%%%%%%%%%%%%%%%%%%%%%%%%%%%%%%%%%%%%%%%%%%%%%%%
\section{Additional experiments}\label{app:additional_experiments}
%%%%%%%%%%%%%%%%%%%%%%%%%%%%%%%%%%%%%%%%%%%%%%%%%%%%%%%%%%%%%%%%%%%%%%%%%%%%%%%%%%%%%%%%%%%%%%%%%%%%%%%

\subsection{Additional comparison of Penalty on the data denoising problem }\label{app:additional_imp_learning}
\subsubsection{Comparison with [\cite{franceschi2017forward}]}
For comparison of Penalty against the RMD-based method presented in [\cite{franceschi2017forward}], we used their setting from Sec.~5.1, which is a smaller version of this data denoising task. For this, we choose a sample of 5000 training, 5000 validation and 10000 test points from MNIST and randomly corrupted labels of 50\% of the training points and used softmax regression in the lower-level of the bilevel formulation (Eq.~(\ref{eq:importance learning})). The accuracy of the classifier trained on a subset of the dataset comprising only of points with importance values greater than 0.9 (as computed by Penalty) along with the validation set is 90.77\%. This is better than the accuracy obtained by Val-only (90.54\%), Train+Val (86.25\%) and the RMD-based method (90.09\%) used by [\cite{franceschi2017forward}] and is close to the accuracy achieved by the Oracle classifier (91.06\%). The bilevel training uses $K=100$ and $T=20$, $\sigma_0$=3, $\rho_0$=0.00001, $\gamma_0$=0.01, $\epsilon_0$=0.01, $\lambda_0$=0.01, $\nu_0$=0.0 with batch-size of 200

\subsubsection{Comparison with [\cite{ren18l2rw}]}
To demonstrate the effectiveness of the penalty in solving the importance learning problem with bigger models, we compared its performance against the recent method proposed by [\cite{ren18l2rw}], which uses a meta-learning approach to find the weights for each example in the noisy training set based on their gradient directions. We use the same setting as their uniform flip experiment with 36\% label noise on the CIFAR10 dataset. We also use our own implementation of the Wide Resnet 28-10 (WRN-28-10) which achieves roughly 93\% accuracy without any label noise. For comparison, we used the validation set of 1000 points and training set of 44000 points with labels of 36\% points corrupted, same as used by [\cite{ren18l2rw}]. We use Penalty with $T=1$ since using larger $T$ was not possible due to extremely high computational needs. However, using a larger value $T$ is expected to improve the results further based on Fig.~\ref{fig:effect_of_T}(a). Different from other experiments in this section we did not use the arctangent conversion to restrict importance values between 0 and 1 but instead just normalize the important values in a batch, similar to the method used by [\cite{ren18l2rw}], for proper comparison. We used $K=200$ and $T=1$, $\sigma_0$=3, $\rho_0$=0.0001, $\gamma_0$=1, $\epsilon_0$=1, $\lambda_0$=10, $\nu_0$=0.0 with batch-size of 75 and used data augmentation during training. We achieve an accuracy of 87.41 $\pm$ 0.26.

\begin{figure}[tb]
\centering
\subfigure[Clean label poisoning attack on dogfish dataset. The top and middle rows show the target and base instances from the test set and the last row shows the poisoned instances obtained from Penalty. Notice that poisoned images (bottom row) are visually indistinguishable from the base images (middle row) and can evade visual detection.]
{
\label{fig:dataset_posioning_dogfish_images}\includegraphics[width=0.9\columnwidth, height=2cm]{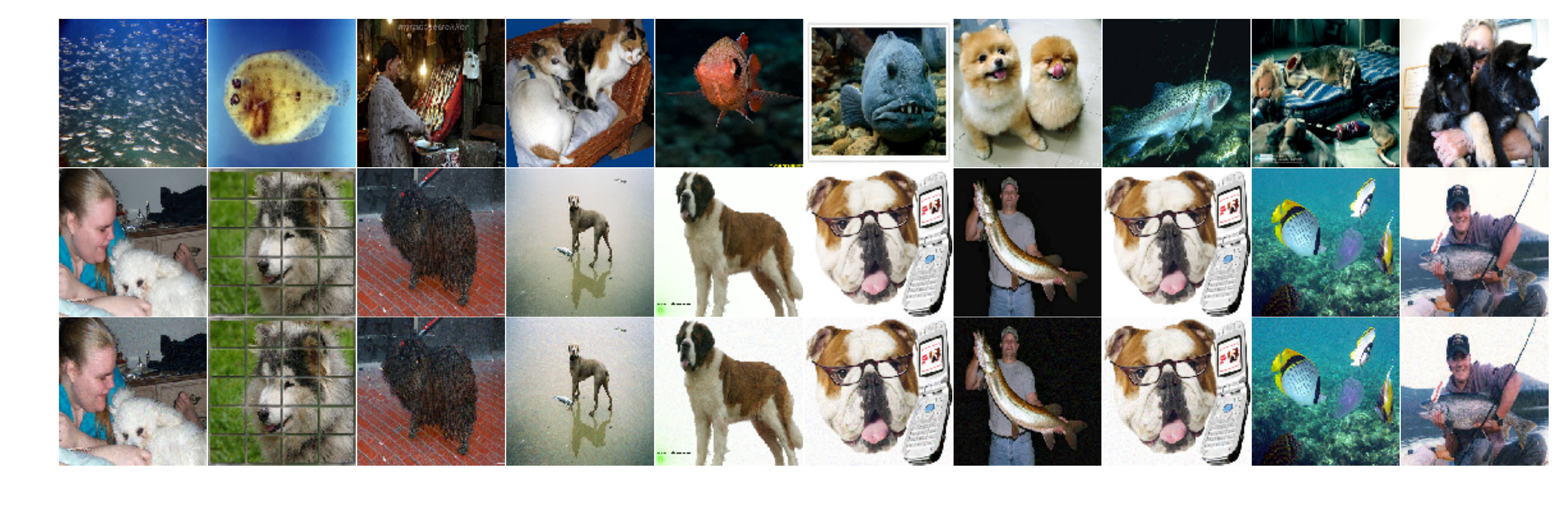}
}
\subfigure[Untargeted data poisoning attack on MNIST. The top row shows the learned poisoned image using Penalty, starting from the images in the bottom row as initial poisoned images. The column number represents the fixed label of the image, i.e. the label of the images in the first column is digit 0, the second column is digit 1, etc. %The average $l_2$ distortion of the images is 5.79.
]
{
\label{fig:dataset_posioning_untargeted}\includegraphics[width=0.9\columnwidth, height=2cm]{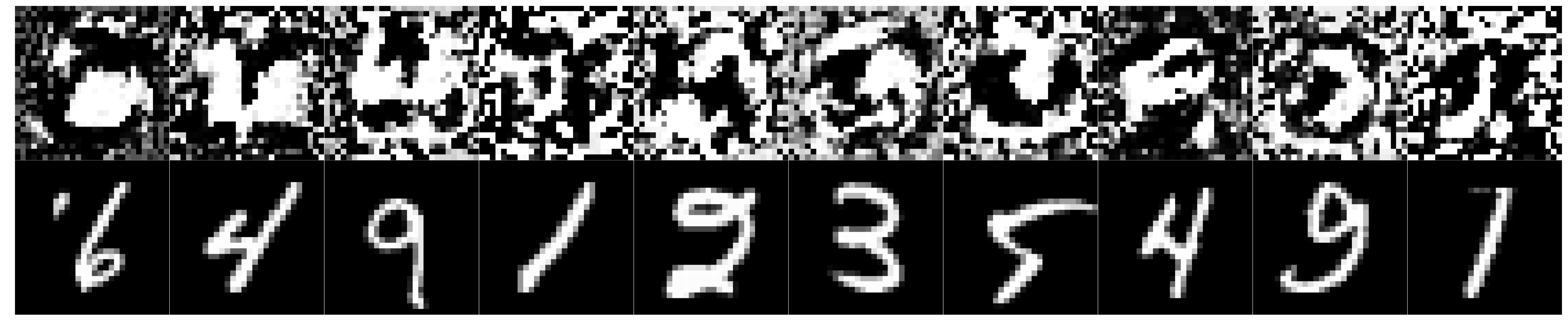}
}
\subfigure[Targeted data poisoning attack on MNIST. The top row shows the learned poisoned images using Penalty, starting from the images in the bottom row as initial poisoned images. Images in the first 5 columns have the fixed label of digit 3, and in the next 5 columns are images with the fixed label of digit 8. %The average $l_2$ distortion of the poisoned examples from base images is 3.89.
]{
\label{fig:dataset_posioning_targeted}\includegraphics[width=0.9\columnwidth, height=2cm]{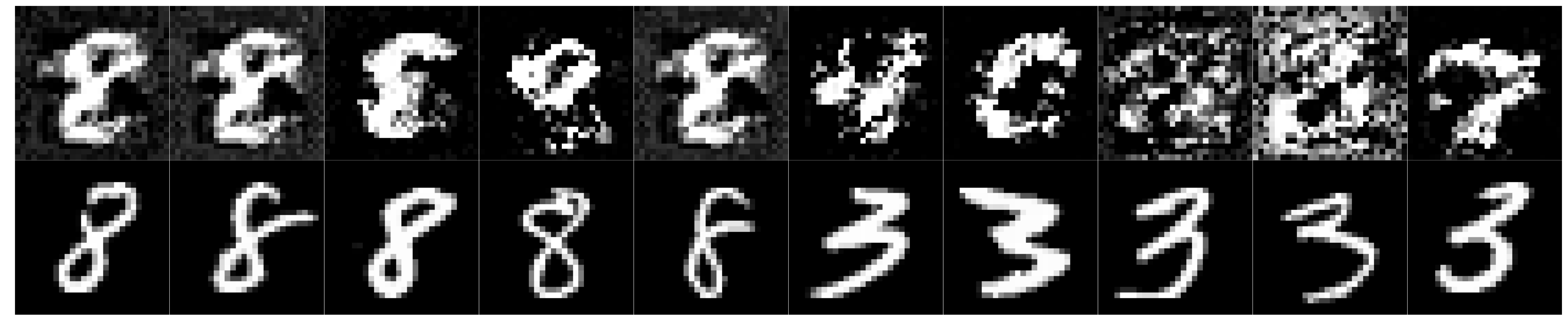}
}
\caption{Poisoning points for clean label and simple data poisoning attacks}
\label{fig:attack_images}
\end{figure}

\begin{table}[tb]
  \caption{Test accuracy (\%) of untargeted poisoning attack (TOP) and success rate (\%) of targeted attack (BOTTOM), using MNIST (Mean $\pm$ s.d. of 5 runs). Results for RMD are from [\cite{munoz2017towards}].
  }
  \label{Table:data_posioning_sc_1}
  \centering
  \small
  \resizebox{\textwidth}{!}{
    \begin{tabular}{c|cccc|cccc}
    \toprule
    & \multicolumn{4}{c}{Untargeted Attacks ({\bf lower} accuracy is better)} & \multicolumn{4}{|c}{Targeted Attacks ({\bf higher} accuracy is better)} \\
    \midrule
         \makecell{Poisoned \\ points}  & \makecell{Label \\ flipping} & \makecell{RMD} & ApproxGrad & Penalty & \makecell{Label \\ flipping} & \makecell{RMD} & ApproxGrad  & Penalty \\
        \midrule
        1\%  & 86.71$\pm$0.32 & 85 & {\bf82.09}$\pm$0.84 & {\bf83.29}$\pm$0.43 & 7.76$\pm$1.07 & 10 & {\bf18.84}$\pm$1.90 & {\bf17.40}$\pm$3.00
        \\
        2\% & 86.23$\pm$ 0.98 & 83 & {\bf77.5}4$\pm$0.57 & {\bf78.14}$\pm$0.53 & 12.08$\pm$2.13 & 15 &  {\bf39.64}$\pm$3.72 & {\bf41.64}$\pm$4.43
        \\
        3\% & 85.17$\pm$0.96 & 82 & {\bf74.41}$\pm$1.14 & {\bf75.14}$\pm$1.09 & 18.36$\pm$1.23 & 25 & {\bf52.76}$\pm$2.69 & {\bf51.40}$\pm$2.72
        \\
        4\% & 84.93$\pm$0.55 & 81 & {\bf71.88}$\pm$0.40 & {\bf72.70}$\pm$0.46 & 24.41$\pm$2.05 & 35 & {\bf60.01}$\pm$1.61 & {\bf61.16}$\pm$1.34
        \\
        5\% & 84.39$\pm$1.06 & 80 & {\bf68.69}$\pm$0.86 & {\bf69.48}$\pm$1.93 & 30.41$\pm$4.24 & - & {\bf65.61}$\pm$4.01 & {\bf65.52}$\pm$2.85
        \\
        6\% & 84.64$\pm$0.69 & 79 & {\bf66.91}$\pm$0.89 & {\bf67.59}$\pm$1.17 & 32.88$\pm$3.47 & - & {\bf71.48}$\pm$4.24 & {\bf70.01}$\pm$2.95
        \\
        \bottomrule
    \end{tabular}
    }
\end{table}

\subsection{Simple data poisoning attack}\label{sec:experiment3_2}
%\subsection{Training-data poisoning}\label{app:experiment3}
Here we discuss a simple data poisoning attack problem that does not involve any constraint on the amount of perturbation on the poisoned data. We solve the following bilevel problem
\begin{equation}\label{eq:poisoning}
\max_{u} \;L_\mathrm{val}(u,w^*) \;\;\mathrm{s.t.}\;\;w^* = \arg\min_{w}\; L_\mathrm{poison}(u,w),
\end{equation}
Here, we evaluate Penalty on the task of generating poisoned training data, such that models trained on this data, perform poorly/differently as compared to the models trained on the clean data. We use the same setting as Sec.~4.2 of [\cite{munoz2017towards}] and test both untargeted and targeted data poisoning on MNIST using the data augmentation technique. We assume regularized logistic regression will be the classifier used for training. The poisoned points obtained after solving Eq.~(\ref{eq:poisoning}) by various methods are added to the clean training set and the performance of a new classifier trained on this data is used to report the results in Table~\ref{Table:data_posioning_sc_1}. 
For untargeted attacks, our aim is to generally lower the performance of the classifier on the clean test set. For this experiment, we select a random subset of 1000 training, 1000 validation, and 8000 testing points from MNIST and initialize the poisoning points with random instances from the training set but assign them incorrect random labels. We use these poisoned points along with clean training data to train logistic regression, in the lower-level problem of Eq.~(\ref{eq:poisoning}). For targeted attacks, we aim to misclassify images of eights as threes. For this, we selected a balanced subset (each of the 10 classes are represented equally in the subset) of 1000 training, 4000 validation, and 5000 testing points from the MNIST dataset. Then we select images of class 8 from the validation set and label them as 3 and use only these images for the upper-level problem in Eq.~(\ref{eq:poisoning}) with a difference that now we want to minimize the error in the upper level instead of maximizing. To evaluate the performance we selected images of 8 from the test set and labeled them as 3 and report the performance on this modified subset of the original test set in the targeted attack section of Table~\ref{Table:data_posioning_sc_1}. For this experiment, the poisoned points are initialized with images of classes 3 and 8 from the training set, with flipped labels. This is because images of threes and eights are the only ones involved in the poisoning. We compare the performance of Penalty against the performance reported using RMD in [\cite{munoz2017towards}] and ApproxGrad. For ApproxGrad, we used 20 lower-level and 20 linear system updates to report the results in Table~\ref{Table:data_posioning_sc_1}. We see that Penalty significantly outperforms the RMD based method and performs similar to ApproxGrad. However, in terms of wall-clock time Penalty has an advantage over ApproxGrad (see Fig.~\ref{fig:effect_of_T}(c) in Appendix~\ref{sec:impact_of_T}). We also compared the methods against a label flipping baseline where we select poisoned points from the validation sets and change their labels (randomly for untargeted attacks and mislabel threes as 8 and eights as 3 for targeted attacks). All bilevel methods are able to beat this baseline showing that solving the bilevel problem generates better poisoning points. Examples of the poisoned points for untargeted and targeted attacks generated by Penalty are shown in Fig.~\ref{fig:attack_images}.
For this experiment, we used $l_2$-regularized logistic regression implemented as a single layer neural network with the cross entropy loss and a weight regularization term with a coefficient of 0.05. The model is trained for 10000 epochs using the Adam optimizer with learning rate of 0.001 for training with and without poisoned data. We pre-train the lower-level with clean training data for 5000 epochs with the Adam optimizer and learning rate 0.001 before starting bilevel training. For untargeted attacks, we optimized Penalty with $K=5000$, $T=20$, $\sigma_0$=0.1, $\rho_0$ = 0.001, $\gamma_0$=10, $\epsilon_0$=1, $\lambda_0$=100, $\nu_0$=0.0. The test accuracy of this model trained on clean data is 87\%. For targeted attack, Penalty is optimized with $K=5000$, $T=20$, $\sigma_0$=0.1, $\rho_0$ = 0.001, $\gamma_0$=10, $\epsilon_0$=1, $\lambda_0$=1, $\nu_0$=0.0. 

\subsection{Impact of T on accuracy and run-time}
\label{sec:impact_of_T}
Here, we compare the accuracy and time for Penalty and ApproxGrad (Fig.~\ref{fig:effect_of_T} and Table~\ref{Table:wall-clock time}) as we vary the number of lower-level iterations $T$ for different experiments. Intuitively, a larger $T$ corresponds to a more accurate approximation of the hypergradient and therefore improves the results for both methods. But this improvement comes with a significant increase in time. Moreover, Fig.~\ref{fig:effect_of_T} shows that relative improvement after $T=20$ is small in comparison to the increased run-time for Penalty and especially for ApproxGrad. Based on these results we used $T=20$ for all our experiments on real data for both methods. The figure also shows that even though Penalty and ApproxGrad have the same linear time complexity (Table~\ref{tbl:complexity}), Penalty is about twice as fast ApproxGrad in wall-clock time on real experiments.

\subsection{Impact of various hyperparameters and terms}\label{app:effect_of_various_hyperparameters}
Here we evaluate the impact of different initial values for the hyperparameters and the impact of different terms added in the modified algorithm (Algorithm \ref{alg:extended}). In particular, we examine the effect of using different initial values of $\lambda_0$ for synthetic experiments and $\lambda_0, \gamma_0$ for untargeted data poisoning with 60 points and also test the effect of having the $\lambda_k g$ and $\gv^T \nu$ (Fig.~\ref{fig:synthetic_exps} and Table \ref{Table:data_posioning_hyperparameter_effects}). Based on the results we find that the initial value of the regularization parameter $\lambda_0$ does not influence the results too much and the absence of $\lambda_k g$ ($\lambda_k$ = 0) also does not change the results too much. We also don't see significant gains from using the augmented Lagrangian term and method of multipliers on these simple problems. However, the initial value of the parameter $\gamma_0$ does influence the results since starting from very large $\gamma_0$ makes the algorithm focus only on satisfying the necessary condition at the lower level ignoring the $f$ whereas with small $\gamma_0$ it can take a large number of iterations for the penalty term to have influence. Apart from these, we also tested the effects of the rate of tolerance decrease ($c_{\epsilon}$) and penalty increase $(c_{\gamma})$, and initial value for $\epsilon_0$. Within certain ranges, the results do not change much.

\begin{table}[tb]
  \caption{\small{Effect of using different initial values for various hyper-parameters with Penalty on untargeted data poisoning attacks, Appendix~\ref{sec:experiment3_2} ({\bf lower} accuracy is better) with 60 poisoning points (Mean $\pm$ s.d. of 5 runs with $T$ = 20 (lower-level iterations)). We used the parameters corresponding to the bold values for the results reported in Table \ref{Table:data_posioning_sc_1}.}
  }
  \label{Table:data_posioning_hyperparameter_effects}
  \centering
  \resizebox{0.6\columnwidth}{!}{
    \begin{tabular}{c|ccccccccccccc}
    \toprule
    \multicolumn{1}{c|}{\makecell{Hyper-\\parameters}} & \multicolumn{12}{c}{Different initial values of various hyperparameters}\\
    \midrule
    \multirow{3}{*}{$\lambda_0$} & \multicolumn{3}{c|}{$\lambda_0$ = 0} & \multicolumn{3}{c|}{$\lambda_0$ = 1} & \multicolumn{3}{c|}{$\lambda_0$ = 10} & \multicolumn{3}{c}{$\lambda_0$ = 100} \\\cline{2-13}
     &\multicolumn{3}{c|}{} & \multicolumn{3}{c|}{} & \multicolumn{3}{c|}{} & \multicolumn{3}{c}{}  \\
     &\multicolumn{3}{c|}{67.87$\pm$1.35} & \multicolumn{3}{c|}{68.21$\pm$1.78} & \multicolumn{3}{c|}{68.18$\pm$1.04} & \multicolumn{3}{c}{{\bf67.59}$\pm$1.17} \\
    \midrule
    
    \multirow{3}{*}{$\nu$} & \multicolumn{6}{c|}{with $\nu$} & \multicolumn{6}{c}{without $\nu$} \\\cline{2-13}
     &\multicolumn{6}{c|}{} & \multicolumn{6}{c}{}  \\
     &\multicolumn{6}{c|}{{\bf67.59}$\pm$1.17} & \multicolumn{6}{c}{68.82$\pm$0.75} \\
    \midrule
    
    \multirow{3}{*}{$\gamma_0$} & \multicolumn{5}{c|}{$\gamma_0$ = 1} & \multicolumn{4}{c|}{$\gamma_0$ = 10} & \multicolumn{3}{c}{$\gamma_0$ = 100} \\\cline{2-13}
     &\multicolumn{5}{c|}{} & \multicolumn{4}{c|}{} & \multicolumn{3}{c}{}  \\
     &\multicolumn{5}{c|}{73.38$\pm$4.98} & \multicolumn{4}{c|}{{\bf67.59}$\pm$1.17} & \multicolumn{3}{c}{71.96$\pm$3.56} \\
    
    \bottomrule
    \end{tabular}
    }
\end{table}
\begin{figure}[tb]
  \centering
  \includegraphics[width=0.45\columnwidth]{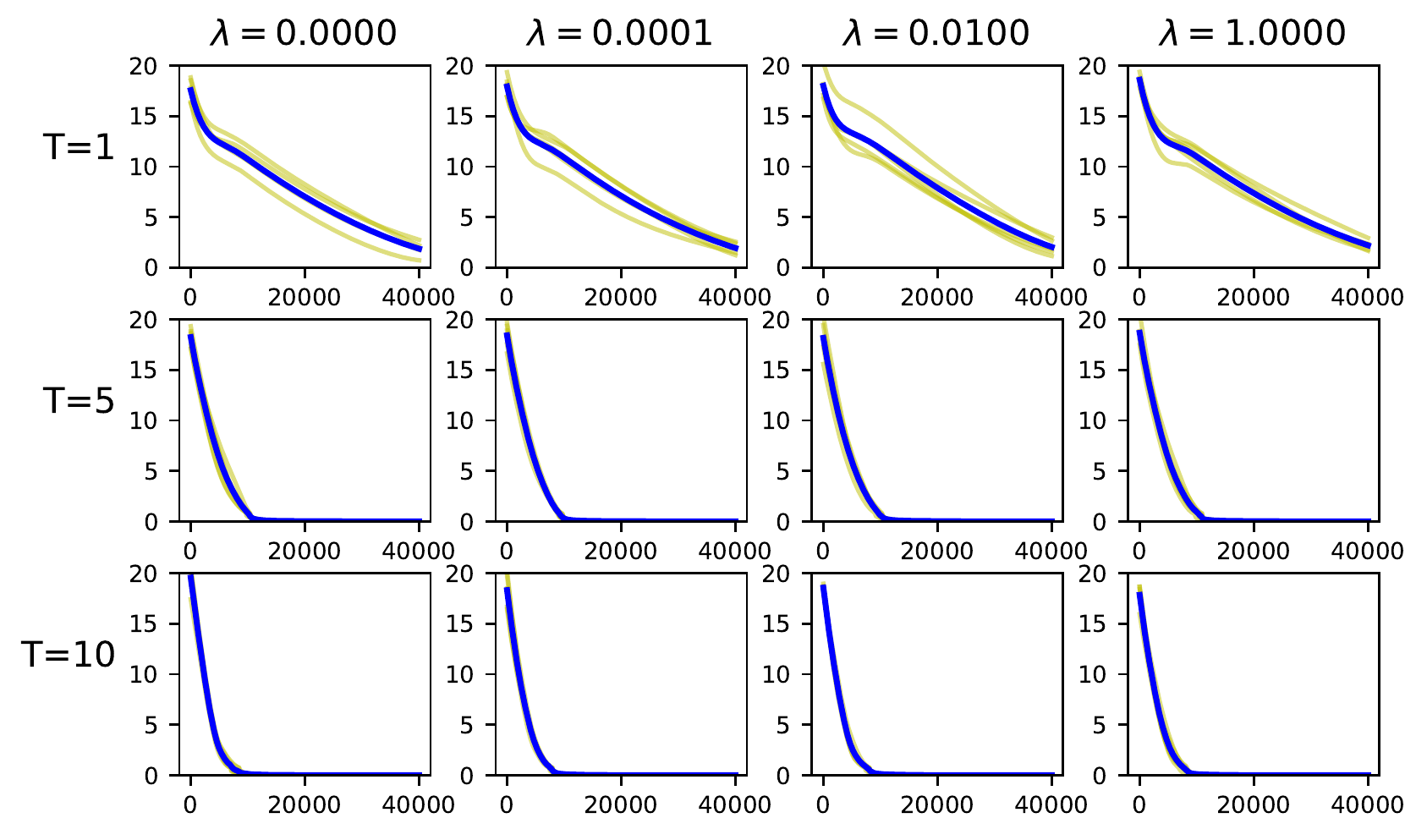}  
  \includegraphics[width=0.45\columnwidth]{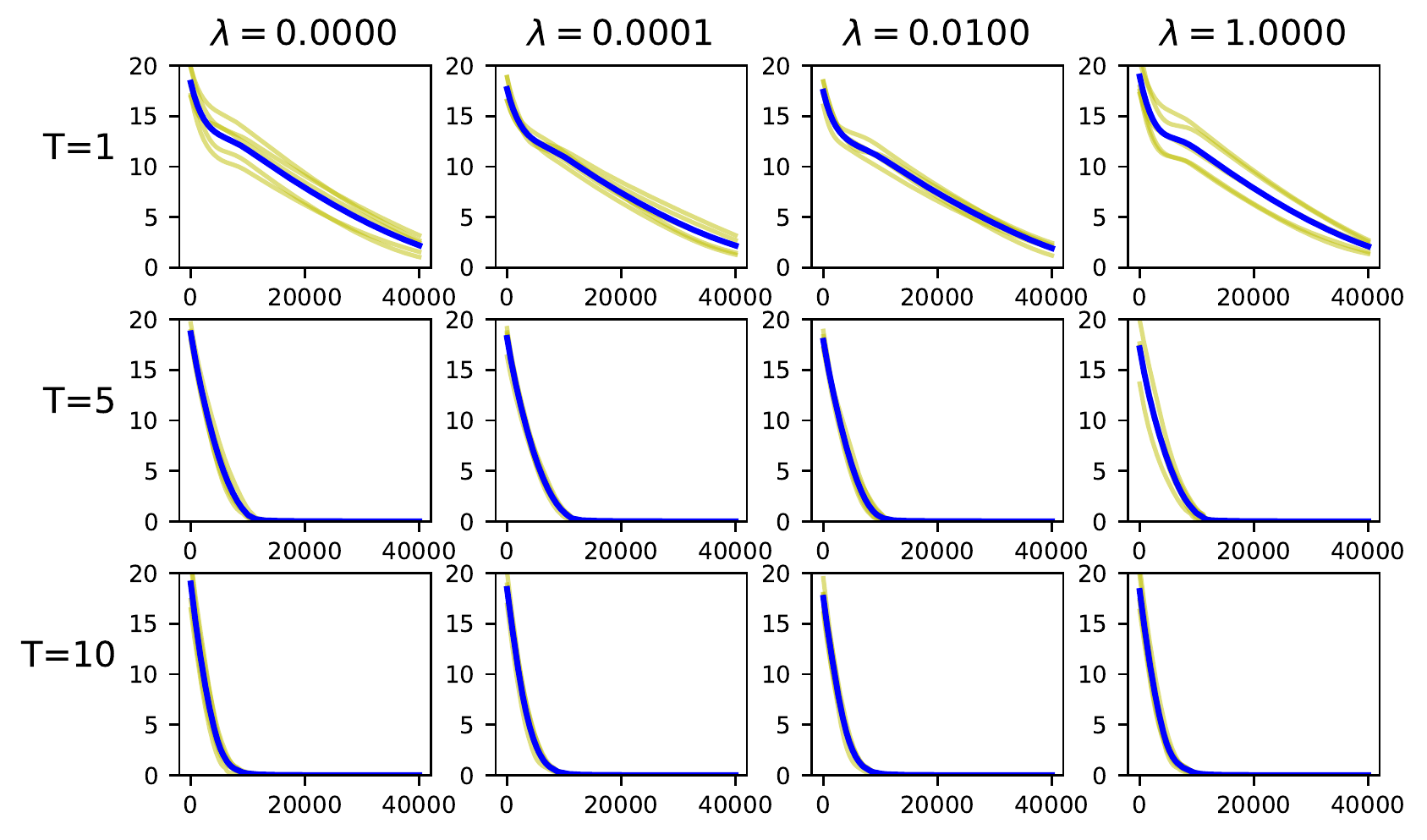}  
  \caption{Penalty method for $T$=1,5,10 and $\lambda_0=0,10^{-4},10^{-2},1$ for Example 1 of Sec.3.1. Top: with $\nu$. Bottom: without $\nu$. Averaged over 5 trials.}
  \label{fig:synthetic_exps}
\end{figure}

\begin{table}[t]
  \caption{\small{Upper- and lower-level variable sizes for different experiments}}
  \label{Table:variable_sizes}
  \centering
  \resizebox{0.6\columnwidth}{!}{
  \begin{tabular}{c|c|c|c}
    \toprule
    Experiment & Dataset & \makecell{Upper-level \\ variable} & \makecell{Lower-level \\ variable} \\
    \midrule
    \multirow{3}{*}{\rotatebox[origin=c]{0}{\makecell{Data \\denoising}}}
    & MNIST & 59K & 1.4M\\%\cline{2-4}
    & CIFAR10 (Alexnet) & 40K & 1.2M \\%\cline{2-4}
    & CIFAR10 (WRN-28-10) & 44K & {\bf36M}  \\
    & SVHN & 72K & 1.3M  \\
    \midrule
    \multirow{2}{*}{\rotatebox[origin=c]{0}{\makecell{Few-shot \\ learning}}}
    & Omniglot & 111K & 39K  \\%\cline{2-4}
    & Mini-Imagenet & {\bf3.8M} & 5K \\    
    \midrule
    \multirow{2}{*}{\rotatebox[origin=c]{0}{\makecell{Data \\ poisoning}}}
    & \makecell{MNIST (Augment 60 \\ poison points)} & 47K  &  8K \\%\cline{2-4}
    & \makecell{ImageNet (Clean \\ label attack)} & 268K &  4K\\
    \bottomrule
  \end{tabular}
  }
\end{table}

%%%%%%%%%%%%%%%%%%%%%%%%%%%%%%%%%%%%%%%%%%%%%%%%%%%%%%%%%%%%%%%%%%%%%%%%%%%%%%%%%%%%%%%%%%%%%%%%%%%%%%%%
\section{Details of the experiments}\label{app:details_experiments}
%%%%%%%%%%%%%%%%%%%%%%%%%%%%%%%%%%%%%%%%%%%%%%%%%%%%%%%%%%%%%%%%%%%%%%%%%%%%%%%%%%%%%%%%%%%%%%%%%%%%%%%
All codes are written in Python using Tensorflow/Keras and were run on Intel CORE i9-7920X CPU with 128 GB of RAM and dual NVIDIA TITAN RTX. Implementation and hyperparameters of the algorithms are experiment-dependent and described separately below.

\subsection{Synthetic problems}\label{app:synthetic examples}
In this experiment, four simple bilevel problems with known optimal solutions are used
to check the convergence of different algorithms. 
The two problems in Fig.~\ref{fig:synthetic1} are
\begin{eqnarray*}
\min_{u,v}\|u\|^2+\|v\|^2, \; \mathrm{s.t.}\; v=\arg\min_v  \|1-u-v\|^2,
\end{eqnarray*}
and 
\begin{eqnarray*}
\min_{u,v} \|v\|^2 - \|u-v\|^2, \; \mathrm{s.t.}\; v=\arg\min_v \|u-v\|^2,
\end{eqnarray*}
where $u=[u_1,\cdots,u_{10}]^T,\;\;|u_i|\leq 5$ and 
$v=[v_1,\cdots,v_{10}]^T,\;\;|v_i|\leq 5$. 
The optimal solutions are $u_i=v_i=0.5,\;\;$ $i=1,\cdots,10$ for the former and $u_i=v_i=0,\;\;$ $i=1,\cdots,10$ for the latter. 
Since there are unique solutions, convergence is measured by the Euclidean distance $\sqrt{\|u-\us\|^2+\|v-\vs\|^2}$
of the current iterate $(u,v)$ and the optimal solution $(\us,\vs)$.

The two problems in Fig.~\ref{fig:synthetic2} are 
\begin{eqnarray*}
\min_{u,v} \|u\|^2 + \|v\|^2, \;\; \mathrm{s.t.} \;\;
v=\arg\min_v (1-u-v)^TA^TA(1-u-v)
\end{eqnarray*}
and 
\begin{eqnarray*}
\min_{u,v} \|v\|^2 - (u-v)^TA^TA(u-v), \;\; \mathrm{s.t.}\;\; v=\arg\min_v (u-v)^TA^TA(u-v),
\end{eqnarray*}
where $A$ is a $5 \times 10$ real matrix such that $A^TA$ is rank-deficient,
and the domains are the same as before. 
These problems are ill-conditioned versions of the previous two problems and
are more challenging. 
The optimal solutions to these two example problems are not unique. 
For the former, the solutions are 
$u=0.5+p$ and $v=0.5+p$ for any vector $p \in \mathrm{Null}(A)$. For the latter, 
$u=p$ and $v=0$ for any vector $p \in \mathrm{Null}(A)$. 
Since they are non-unique, convergence is measured by the residual distance $\sqrt{\|P(u-0.5)\|^2 + \|P(v-0.5)\|^2}$ for the former
and $\sqrt{\|Pu\|^2 + \|v\|^2}$ for the latter, where $P = A^T(AA^T)^{-1}A$
is the orthogonal projection to the row-space of $A$. 

The algorithms used in this experiment are GD, RMD, ApproxGrad, and Penalty.
Adam optimizer is used for minimization everywhere except RMD which uses
gradient descent for a simpler implementation. 
The learning rates common to all algorithms are $\sigma_0=10^{-3}$ for $u$-update
and $\rho_0 = 10^{-4}$ for $v$- and $p$-updates.
For Penalty, the values $\gamma_0=1$, $\lambda_0=10$, and $\epsilon_0=1$ are used.
For each problem and algorithm, 20 independent trials are performed with random initial locations $(u_0,v_0)$ sampled uniformly in the domain, 
and random entries of $A$ sampled from independent Gaussian distributions.
We test with $T=1,5,10$. 
Each run was stopped after $K=40000$ iterations of $u$-updates.

\subsection{Data denoising by importance learning}\label{app:experiment1}
Following the formulation for data denoising presented in Eq.~(\ref{eq:importance learning}), we associate an importance value (denoted by $u_i$) with each point in the training data. Our goal is to find the correct values for these $u_i$'s such that the noisy points are given lower importance values and clean points are given higher importance values. In our experiments, we allow the importance values to be between 0 and 1. We use the change of variable technique to achieve this. We set $u'_i = 0.5 (tanh(u_i) + 1)$ and since $-1 \leq tanh(u_i) \leq 1$, $u'_i$ is automatically scaled between 0 and 1. We use a warm start for the bilevel methods (Penalty and ApproxGrad) by pre-training the network using the validation set and initializing the importance values with the predicted output probability from the pre-trained network. We see an advantage in the convergence speed of the bilevel methods with this pre-training. Below we describe the network architectures used for our experiments.

For the experiments on the MNIST dataset, our network consists of a convolution layer with a kernel size of 5x5, 64 filters, and ReLU activation, followed by a max-pooling layer of size 2x2 and a dropout layer with a drop rate of 0.25. This is followed by another convolution layer with a 5x5 kernel, 128 filters, and ReLU activation followed by similar max pooling and dropout layers. Then we have 2 fully connected layers with ReLU activation of sizes 512 and 256 respectively, each followed by a dropout layer with a drop rate of 0.5. Lastly, we have a softmax layer with 10 classes. We used the Adam optimizer with a learning rate of 0.00001, batch size of 200, and 100 epochs to report the accuracy of Oracle, Val-Only, and Train+Val classifiers. For bilevel training using Penalty we used $K=100$, $T=20$, $\sigma_0$=3, $\rho_0$=0.00001, $\gamma_0$=0.01, $\epsilon_0$=0.01, $\lambda_0$=0.01, $\nu_0$=0.000001 as per Alg.~\ref{alg:extended}.

For the experiments on the CIFAR10 dataset, our network consists of 3 convolution blocks with filter sizes of 48, 96, and 192. Each convolution block consists of two convolution layers, each with a kernel size of 3x3 and ReLU activation. This is followed by a max-pooling layer of size 2x2 and a drop-out layer with a drop rate of 0.25. After these 3 blocks, we have 2 dense layers with ReLU activation of sizes 512 and 256 respectively, each followed by a dropout layer with a rate of 0.5. Finally, we have a softmax layer with 10 classes. This is optimized with the Adam optimizer using a learning rate of 0.001 for 200 epochs with a batch size of 200 to report the accuracy of Oracle, Val-Only, and Train+Val classifiers. For this experiment, we used data augmentation during our training. For the bilevel training using Penalty we used $K=200$, $T=20$, $\sigma_0$=3, $\rho_0$=0.00001, $\gamma_0$=0.01, $\epsilon_0$=0.01, $\lambda_0$=0.01, $\nu_0$=0.0001 with mini-batches of size 200. We also use data augmentation for bilevel training.

For the experiments on the SVHN dataset, our network consists of 3 blocks each with 2 convolution layers with a kernel size of 3x3 and ReLU activation followed by a max-pooling and drop out layer (drop rate = 0.3). The two convolution layers of the first block have 32 filters, the second block has 64 filters and the last block has 128 filters. This is followed by a dense layer of size 512 with ReLU activation and a dropout layer with a drop rate = 0.3. Finally, we have a softmax layer with 10 classes. This is optimized with the Adam optimizer and learning rate of 0.001 for 100 epochs to report results of Oracle, Val-Only, and Train+Val classifiers. The bilevel training uses $K=100$ and $T=20$, $\sigma_0$=3, $\rho_0$=0.00001, $\gamma_0$=0.01, $\epsilon_0$=0.01, $\lambda_0$=0.01, $\nu_0$=0.0 with batch-size of 200. 
The test accuracy of these models, when trained on the entire training data without any label corruption, is 99.5\% for MNIST, 86.2\% for CIFAR10, and 91.23\% for SVHN. 
For all the experiments with ApproxGrad, we used 20 updates for the lower-level and 20 updates for the linear system and did the same number of epochs as for Penalty (i.e. 100 for MNIST and SVHN and 200 for CIFAR), with a mini-batch-size 200.

\subsection{Few-shot learning}\label{app:experiment2}
For these experiments, we used the Omniglot [\cite{lake2015human}] dataset consisting of 20 instances (size 28 $\times$ 28) of 1623 characters from 50 different alphabets and the Mini-ImageNet [\cite{vinyals2016matching}] dataset consisting of 60000 images (size 84 $\times$ 84) from 100 different classes of the ImageNet [\cite{deng2009imagenet}] dataset. For the experiments on the Omniglot dataset, we used a network with 4 convolution layers to learn the common representation for the tasks. The first three layers of the network have 64 filters, batch normalization, ReLU activation, and a 2 $\times$ 2 max-pooling. The final layer is the same as the previous ones with the exception that it does not have any activation function. The final representation size is 64. For the Mini-ImageNet experiments, we used a residual network with 4 residual blocks consisting of 64, 96, 128, and 256 filters followed by a 1 $\times$ 1 convolution block with 2048 filters, average pooling, and finally a 1 $\times$ 1 convolution block with 384 filters. Each residual block consists of 3 blocks of 1 $\times$ 1 convolution, batch normalization, leaky ReLU with leak = 0.1, before the residual connection and is followed by dropout with rate = 0.9. The last convolution block does not have any activation function. The final representation size is 384. Similar architectures have been used ite{franceschi2018bilevel} in their work with the difference that we don't use any activation function in the last layers of the representation in our experiments. For both the datasets, the lower-level problem is a softmax regression with a difference that we normalize the dot product of the input representation and the weights with the $l_2$-norm of the weights and the $l_2$-norm of the input representation, similar to the cosine normalization proposed by [\cite{luo2018cosine}]. For $N$ way classification, the dimension of the weights in the lower-level are 64 $\times$ $N$ for Omniglot and 384 $\times$ $N$ for Mini-ImageNet.
For our Omniglot experiments we use a meta-batch-size 30 for 5-way and 20-way classification and a meta-batch-size of 2 for 5-way classification with Mini-ImageNet. We use $T=20$ iterations for the lower-level in all experiments and ran them for $K$=10000. The hyper-parameters used for Penalty are $\sigma_0$=0.001, $\rho_0$=0.001, $\gamma_0$=0.01, $\epsilon_0$=0.01, $\lambda_0$=0.01, $\nu_0$=0.0001. 

\subsection{Clean label data poisoning attack}\label{app:clean-label-full}
We solve the following problem for clean label poisoning:
\begin{equation}\label{eq:clean-label-poisoning-full}
    \min_{u} \;L_\mathrm{t}(u,w^*) + \|r(t) - r(u)\| \;\;\mathrm{s.t.} \;\; \|x_\mathrm{base} - u\|_2 \leq \epsilon\;\mathrm{and}\;w^* = \arg\min_{w}\; L_\mathrm{poison}(u,w),
\end{equation}
We use the dog vs. fish image dataset as used by [\cite{koh2017understanding}], consisting of 900 training and 300 testing examples from each of the two classes. The size of the images in the dataset is 299 $\times$ 299 with pixel values scaled between -1 and 1. Following the setting in Sec.~5.2 of [\cite{koh2017understanding}], we use the InceptionV3 model with weights pre-trained on ImageNet. We train a dense layer on top of these pre-trained features using the RMSProp optimizer and a learning rate of 0.001 optimized for 1000 epochs before starting bilevel training. Test accuracy obtained with training on clean training data is 98.33. We repeat the same procedure as training during evaluation and train the dense layer with training data augmented with a poisoned point. For solving the Eq.~(\ref{eq:clean-label-poisoning-full}) with Penalty we converted the inequality constraint to an equality constraint by adding a non-negative slack variable. Penalty is optimized with $K=200$, $T=10$, $\sigma_0$=0.01, $\rho_0$ = 0.001, $\gamma_0$=1, $\epsilon_0$=1, $\lambda_0$=1. 

The experiment shown in Fig.~\ref{fig:dataset_posioning_dogfish} is done on the correctly classified instances from the test set. For a fair comparison with Alg. 1 in [\cite{shafahi2018poison}] we choose the same target and base instance for both the algorithms and generate the poison points. We modify Alg. 1 of [\cite{shafahi2018poison}] in order to constrain the amount of perturbation it adds to the base image to generate the poison point. We achieve this by projecting the perturbation back onto the $l_2$ ball of radius $\epsilon$ whenever it exceeds. This is a standard trick used by several methods which generate adversarial examples for test time attacks. We use $\beta = 0.1, \;\lambda = 0.01$ for Alg. 1 of [\cite{shafahi2018poison}] and run it for 2000 epochs in this experiment. For both the algorithms we aim to find the smallest $\epsilon$ that causes misclassification. We incrementally search for the $\epsilon \in \{1,2, ..., 16\}$ and record the minimum one that misclassifies the particular target. These are then used to report the average distortion in Fig.~\ref{fig:dataset_posioning_dogfish}.

\end{document}